\documentclass[8pt]{article}
\usepackage[margin=1in]{geometry}
\usepackage{amsfonts,amsthm,amsmath,amssymb}
\newtheorem{theorem}{Theorem}[section]

\newtheorem{proposition}[theorem]{Proposition}
\usepackage{url}            % simple URL typesetting
\usepackage{booktabs}       % professional-quality tables
\usepackage{amsfonts}       % blackboard math symbols
\usepackage{nicefrac}       % compact symbols for 1/2, etc.
\usepackage{microtype}      % microtypography
\usepackage{graphicx}
\usepackage{subcaption}
\usepackage[boxed,commentsnumbered,ruled]{algorithm2e}
\usepackage{color}
\usepackage[utf8]{inputenc}
\usepackage{authblk}
\usepackage{float}
\usepackage[T1]{fontenc}
\usepackage{lmodern}
\usepackage{natbib}

\author[1,2]{Yao Lu} 
\author[3,4]{Jukka Corander}
\author[1,5]{Zhirong Yang\footnote{Corresponding author. Email: \texttt{zhirong.yang@ntnu.no}}}

\affil[1]{Department of Computer Science, Aalto University}
\affil[2]{College of Engineering and Computer Science, Australian National University}
\affil[3]{Department of Department of Biostatistics, University of Oslo}
\affil[4]{Department of Mathematics and Statistics, University of Helsinki}
\affil[5]{Department of Computer Science, Norwegian University of Science and Technology}

\title{Doubly Stochastic Neighbor Embedding on Spheres}
\date{}

\begin{document}
\maketitle

\def\cvs{${[}$Id: macros.tex,v 1.1 2012-02-22 07:18:00 rozyang Exp ${]}$}

\newcommand{\matA}{\mathbf{A}}
\newcommand{\matB}{\mathbf{B}}
\newcommand{\matC}{\mathbf{C}}
\newcommand{\matD}{\mathbf{D}}
\newcommand{\matE}{\mathbf{E}}
\newcommand{\matF}{\mathbf{F}}
\newcommand{\matG}{\mathbf{G}}
\newcommand{\matH}{\mathbf{H}}
\newcommand{\matI}{\mathbf{I}}
\newcommand{\matK}{\mathbf{K}}
\newcommand{\matL}{\mathbf{L}}
\newcommand{\matM}{\mathbf{M}}
\newcommand{\matN}{\mathbf{N}}
\newcommand{\matO}{\mathbf{O}}
\newcommand{\matP}{\mathbf{P}}
\newcommand{\matQ}{\mathbf{Q}}
\newcommand{\matR}{\mathbf{R}}
\newcommand{\matS}{\mathbf{S}}
\newcommand{\matT}{\mathbf{T}}
\newcommand{\matU}{\mathbf{U}}
\newcommand{\matV}{\mathbf{V}}
\newcommand{\matW}{\mathbf{W}}
\newcommand{\matX}{\mathbf{X}}
\newcommand{\matY}{\mathbf{Y}}
\newcommand{\matZ}{\mathbf{Z}}
\newcommand{\matg}{\mathbf{g}}

\newcommand{\calA}{\mathcal{A}}
\newcommand{\calB}{\mathcal{B}}
\newcommand{\calC}{\mathcal{C}}
\newcommand{\calD}{\mathcal{D}}
\newcommand{\calE}{\mathcal{E}}
\newcommand{\calF}{\mathcal{F}}
\newcommand{\calG}{\mathcal{G}}
\newcommand{\calH}{\mathcal{H}}
\newcommand{\calI}{\mathcal{I}}
\newcommand{\calJ}{\mathcal{J}}
\newcommand{\calK}{\mathcal{K}}
\newcommand{\calL}{\mathcal{L}}
\newcommand{\calM}{\mathcal{M}}
\newcommand{\calN}{\mathcal{N}}
\newcommand{\calO}{\mathcal{O}}
\newcommand{\calP}{\mathcal{P}}
\newcommand{\calQ}{\mathcal{Q}}
\newcommand{\calR}{\mathcal{R}}
\newcommand{\calS}{\mathcal{S}}
\newcommand{\calT}{\mathcal{T}}
\newcommand{\calU}{\mathcal{U}}
\newcommand{\calV}{\mathcal{V}}
\newcommand{\calW}{\mathcal{W}}
\newcommand{\calX}{\mathcal{X}}
\newcommand{\calY}{\mathcal{Y}}
\newcommand{\calZ}{\mathcal{Z}}

\newcommand{\bbR}{\mathbb{R}}
\newcommand{\bbZ}{\mathbb{Z}}
\newcommand{\bbE}{\mathbb{E}}
\newcommand{\bbH}{\mathbb{H}}
\newcommand{\bbS}{\mathbb{S}}

\newcommand{\veca}{\mathbf{a}}
\newcommand{\vecb}{\mathbf{b}}
\newcommand{\vecc}{\mathbf{c}}
\newcommand{\vecd}{\mathbf{d}}
\newcommand{\vece}{\mathbf{e}}
\newcommand{\vecf}{\mathbf{f}}
\newcommand{\vecg}{\mathbf{g}}
\newcommand{\vech}{\mathbf{h}}
\newcommand{\veci}{\mathbf{i}}
\newcommand{\vecj}{\mathbf{j}}
\newcommand{\veck}{\mathbf{k}}
\newcommand{\vecl}{\mathbf{l}}
\newcommand{\vecm}{\mathbf{m}}
\newcommand{\vecn}{\mathbf{n}}
\newcommand{\veco}{\mathbf{o}}
\newcommand{\vecp}{\mathbf{p}}
\newcommand{\vecq}{\mathbf{q}}
\newcommand{\vecr}{\mathbf{r}}
\newcommand{\vecs}{\mathbf{s}}
\newcommand{\vect}{\mathbf{t}}
\newcommand{\vecu}{\mathbf{u}}
\newcommand{\vecv}{\mathbf{v}}
\newcommand{\vecw}{\mathbf{w}}
\newcommand{\vecx}{\mathbf{x}}
\newcommand{\vecy}{\mathbf{y}}
\newcommand{\vecz}{\mathbf{z}}

\newcommand{\vecalpha}{\boldsymbol{\alpha}}
\newcommand{\vecbeta}{\boldsymbol{\beta}}
\newcommand{\veceta}{\boldsymbol{\eta}}
\newcommand{\vectheta}{\boldsymbol{\theta}}
\newcommand{\vecphi}{\boldsymbol{\phi}}
\newcommand{\vecpsi}{\boldsymbol{\psi}}
\newcommand{\vecrho}{\boldsymbol{\rho}}
\newcommand{\vectau}{\boldsymbol{\tau}}
\newcommand{\vecmu}{\boldsymbol{\mu}}
\newcommand{\veceps}{\boldsymbol{\epsilon}}
\newcommand{\vecxi}{\boldsymbol{\xi}}
\newcommand{\vecPhi}{\boldsymbol{\Phi}}
\newcommand{\vecDelta}{\boldsymbol{\Delta}}

\newcommand{\matDelta}{\boldsymbol{\Delta}}
\newcommand{\matEta}{\boldsymbol{\eta}}
\newcommand{\matOmega}{\boldsymbol{\Omega}}
\newcommand{\matPhi}{\boldsymbol{\Phi}}
\newcommand{\matPsi}{\boldsymbol{\Psi}}
\newcommand{\matTheta}{\boldsymbol{\Theta}}
\newcommand{\matLambda}{\boldsymbol{\Lambda}}
\newcommand{\matSigma}{\boldsymbol{\Sigma}}
\newcommand{\matzero}{\mathbf{0}}
\newcommand{\IndexSetI}{\mathcal{I}}
\newcommand{\grad}{\mathcal{\nabla}}

\newcommand{\vecone}{\mathbf{1}}
\newcommand{\veczero}{\mathbf{0}}

\def\maximize{\mathop{{\mathgroup\symoperators maximize}}}
\def\Maximize{\mathop{{\mathgroup\symoperators Maximize}}}
\def\minimize{\mathop{{\mathgroup\symoperators minimize}}}

\def\approach{\mathop{{\mathgroup\symoperators \longrightarrow}}}
\def\defineoperator{\mathop{{\mathgroup\symoperators =}}}
\newcommand{\define}{\defineoperator^{\text{def}}}

\newcommand{\Tr}{\text{Tr}}
\newcommand{\trace}{\text{trace}}
\newcommand{\diag}{\text{diag}}
\newcommand{\gradWJ}{\nabla_{\scriptscriptstyle{\matW}}\calJ}
\newcommand{\const}{\text{constant}}
\newcommand{\fracpartial}[2]{\frac{\partial #1}{\partial  #2}}

\newcommand{\defeq}{\stackrel{\text{def}}{=}}

\begin{abstract}
Stochastic Neighbor Embedding (SNE) methods minimize the divergence between the similarity matrix of a high-dimensional data set and its counterpart from a low-dimensional embedding, leading to widely applied tools for data visualization. Despite their popularity, the current SNE methods experience a crowding problem when the data include highly imbalanced similarities. This implies that the data points with higher total similarity tend to get crowded around the display center. To solve this problem, we introduce a fast normalization method and normalize the similarity matrix to be doubly stochastic such that all the data points have equal total similarities. Furthermore, we show empirically and theoretically that the doubly stochasticity constraint often leads to embeddings which are approximately spherical. This suggests replacing a flat space with spheres as the embedding space. The spherical embedding eliminates the discrepancy between the center and the periphery in visualization, which efficiently resolves the crowding problem. We compared the proposed method (DOSNES) with the state-of-the-art SNE method on three real-world datasets and the results clearly indicate that our method is more favorable in terms of visualization quality. DOSNES is freely available at \url{http://yaolubrain.github.io/dosnes/}.
\end{abstract}

\section{Introduction}

{\let\thefootnote\relax\footnote{{The paper is under consideration at Pattern Recognition Letters.}}}

Information visualization by dimensionality reduction facilitates a
viewer to quickly digest information in massive data. It is therefore
increasingly applied as a critical component in scientific research,
digital libraries, data mining, financial data analysis, market
studies, manufacturing production control and drug discovery, etc. Numerous dimensionality reduction methods
have been introduced, ranging from linear methods 
such as Principal Component Analysis %\citep{jolliffe2002principal},
to nonlinear methods such as
Multidimensional Scaling (MDS), \citep[MDS;][]{torgerson1952multidimensional},
Isomap \citep{tenenbaum2000global},
Locally Linear Embedding \citep{roweis2000nonlinear},
Curvilinear Component Analysis \citep{demartines1997curvilinear},
Laplacian Eigenmaps \citep{belkin2001laplacian}, and
Gaussian Process Latent Variable Models \citep{lawrence2004gaussian}.
%
%Joint Optimization of Fidelity and Commensurability \citep{lyzinski2017fast} and
%
%visualization of tree ensembles by Partition Maps \citep{meinshausen2011partition} and of clustered functional observations in reduced subspace \citep{gattone2012clustering}.
%Maximum Variance Unfolding \cite{mvu}.
%
A survey on nonlinear dimensionality reduction has been given by \citet{van2009dimensionality}.
Aspects in Multidimensional Scaling are discussed by \citet{buja2009mds}.

Recently, Stochastic Neighbor Embedding (SNE) and its variants \citep{hinton2002stochastic,van2008visualizing,van2014accelerating,ICML2014,sun2015spacetime,Tang2016www}
%\citep{hinton2002stochastic,van2008visualizing,ICML2014,Tang2016www}
have achieved remarkable progress in data visualization, especially for displaying clusters in data. 
An SNE method takes as input the pairwise similarities between data points in the high-dimensional space and tries to preserve the similarities in a low-dimensional space by minimizing the Kullback-Leibler divergence between the input and output similarity matrices.

The input to SNE is a similarity matrix or the affinity matrix of a weighted graph. When the node degrees of the graph are highly imbalanced, SNE tends to place the high-degree nodes in the center and the low-degree ones in the periphery, regardless of the intrinsic similarities between the nodes. Therefore, SNE often experiences the ``crowding-in-the-center'' problem for highly imbalanced affinity graphs.

We propose two techniques to overcome the above-mentioned drawback. First, we impose a doubly stochasticity constraint on the input similarity matrix. Two-way normalization has been shown to improve spectral clustering \citep{zass2006doubly} and here we verify that it is also beneficial for data visualization. Moreover, if the neighborhood graph is asymmetric, for example, $k$-Nearest-Neighbors ($k$NN) or entropy affinities \citep{van2008visualizing,vladymyrov2013ea}, we provide an efficient method for converting it to a doubly stochastic matrix.

Second, we observe that the data points are often distributed approximately around a sphere if the input similarity matrix is doubly stochastic, and we provide a theoretical analysis of this phenomenon. Our analysis suggests replacing the two-dimensional Euclidean embedding space with spheres in the three-dimensional space. Since there is no global center or periphery on the sphere geometry, the visualization is then naturally free of ``crowding-in-the-center'' problem. Moreover, we present an efficient projection step for adapting an SNE method with the spherical constraint.

We tested the proposed method on several real-world datasets and compared it with the state-of-the-art SNE method, t-SNE \citep{van2008visualizing}. The new method is superior to t-SNE in resolving the crowding problem and in preserving intrinsic similarities.

In the next section we briefly review SNE methods. We then discuss doubly stochastic similarity matrix and  spherical embedding in Sections \ref{sec:dssm} and \ref{sec:sedssm}, respectively.
%The related work is reviewed in Section \ref{sec:rw}.
We present experimental results in Section \ref{sec:exp} and conclusions in Section \ref{sec:conclusion}.

\section{Stochastic Neighbor Embedding} \label{sec:ne}

Stochastic Neighbor Embedding \citep[SNE;][]{hinton2002stochastic} is a nonlinear dimensionality reduction method.  Given a set of multivariate data points
$\{x_1,x_2,\dots,x_n\}$, where $x_i\in\bbR^D$, their neighborhood is
encoded in a square nonnegative matrix $P$, where $P_{ij}$ is
the probability that $x_j$ is a neighbor of
$x_i$. SNE finds a mapping $x_i\mapsto y_i\in\bbR^d$
for $i=1,\dots,n$ such that the neighborhoods are approximately
preserved in the mapped space. Usually  the mapping is defined such that $d=2$ or $3$, and $d<D$. If
the neighborhood in the mapped space is encoded in $Q\in\bbR^{n\times
	n}$ such that $Q_{ij}$ is the probability that $y_j$
is a neighbor of $y_i$, the SNE task is to minimize the Kullback-Leibler divergence $\calD_\text{KL}(P||Q)$ over
$Y=\left[y_1,y_2,\dots,y_n\right]^T$.

Symmetric Stochastic Neighbor Embedding \citep[s-SNE;][]{van2008visualizing} is a variant of SNE. Given input similarity $p_{ij}\geq0$, s-SNE minimizes Kullback-Leibler divergence between the matrix-wise normalized similarities $P_{ij}=p_{ij}/\sum_{ab}p_{ab}$ and $Q_{ij}=q_{ij}/\sum_{ab}q_{ab}$. The output similarity $q_{ij}$ is typically chosen to be 
proportional to a Gaussian distribution
so that $q_{ij}=\exp\left(-\|y_i-y_j\|^2\right)$,
or proportional to a Cauchy distribution
so that $q_{ij}=(1+\|y_i-y_j\|^2)^{-1}$. The Cauchy s-SNE method is also called t-Distributed Stochastic Neighbor Embedding \citep[t-SNE;][]{van2008visualizing}.
The optimization of s-SNE can be implemented with the gradients for Gaussian case:
$\partial\calJ/\partial y_i=4\sum_j(P_{ij}-Q_{ij})(y_i-y_j)$ and for Cauchy case $\partial\calJ/\partial y_i=4\sum_j(P_{ij}-Q_{ij})(y_i-y_j)q_{ij}$. Here $4\sum_jP_{ij}(y_i-y_j)$ or $4\sum_jP_{ij}(y_i-y_j)q_{ij}$ can be interpreted as the attractive force for $y_i$, while $-4\sum_jQ_{ij}(y_i-y_j)$ or $-4\sum_jQ_{ij}(y_i-y_j)q_{ij}$ as the repulsive force.

\section{Doubly Stochastic Similarity Matrix} \label{sec:dssm}
The input to s-SNE, $P$, is a nonnegative and symmetric matrix and can be treated as the affinity matrix of an undirected weighted graph. If the degree (i.e. row sum or column sum of $P$) distribution of nodes is highly non-uniform, then the high-degree  nodes will usually receive and emit more attractive force than the average nodes during the iterative learning. As a result, these nodes often glue together and form the center of display. On the other hand, the low-degree nodes tend to be placed in the periphery due to less attraction. This behavior is often undesired in visualization because it only reveals the data centrality but hinders the discovery of other useful patterns, and may be directly misleading when some high-degree nodes are actually disconnected in the underlying data.

To overcome the above drawback, we can normalize the graph affinity such that the nodes have the same degree. For undirected graphs, this can be implemented by replacing the unitary matrix-wise sum constraint $\sum_{ij}P_{ij}=1$ in s-SNE with the doubly stochasticity constraint, i.e.~$\sum_iP_{ij}=\sum_jP_{ij}=1$.

Given a non-normalized matrix, we can apply Sinkhorn-Knopp \citep{sinkhorn1967concerning} or Zass-Shashua method \citep{zass2006doubly} to project it to the closest doubly stochastic matrix $P$. In this work we use the former because it can maintain the sparsity of in the similarity matrix, which is often needed for large-scale tasks. Given a non-normalized similarity matrix $S$, the Sinkhorn-Knopp method initializes $P=S$ and iterates the following update rules until $P$ has converged:
%for all $i$, $j$, $P_{ij}\leftarrow P_{ij}u_i^{-1/2}u_j^{-1/2}$, where $u_i\leftarrow\sum_jP_{ij}$ for all $i$.
%for all $i,j$,
%\begin{align}
%P_{ij}\leftarrow \frac{P_{ij}}{\sum_u P_{iu}}, \\
%P_{ij}\leftarrow \frac{P_{ij}}{\sum_v P_{vj}}.
%\end{align}
%P_{ij}\leftarrow \left(\sum_tP_{it}\right)^{-1/2}P_{ij}\left(\sum_tP_{jt}\right)^{-1/2} \\
%
\begin{align}
&\text{for all }i,   ~~u_i\leftarrow\sum_jP_{ij},\\
&\text{for all }i,j, ~~P_{ij}\leftarrow P_{ij}u_i^{-1/2}u_j^{-1/2}.
\end{align}

Alternatively, the neighborhood information in high-dimensional space can be encoded in an asymmetric matrix $B\geq0$ with $n$ rows, for example, the $k$NN graph or the entropy affinities \citep{van2008visualizing,vladymyrov2013ea}. $B$ can also be a non-square dyadic data such as document-term or author-paper co-occurrence matrix. In these cases, we can apply the following steps to construct a doubly stochastic matrix: suppose $\sum_kB_{ik}>0$ for all $i$, we first calculate
for all $i,k$,
\begin{align}
\label{eq:dcd1}
%$A_{ik}\leftarrow B_{ik} / \sum_uB_{iu}$,
A_{ik}&\leftarrow \frac{B_{ik}}{\sum_uB_{iu}},
\end{align}
\vspace{-3mm}
and then for all $i,j$
\begin{align}
\label{eq:dcd2}
P_{ij}&\leftarrow \sum_k\frac{A_{ik}A_{jk}}{\sum_vA_{vk}}.
\end{align}
It is easy to verify that by this construction $P$ is symmetric and doubly stochastic.
The calculations of $A$ and $P$ 
%Eqs.~\ref{eq:dcd1} and \ref{eq:dcd2}
are performed only once and are thus computationally much more efficient than Sinkhorn-Knopp method which needs iterative steps. Here the matrix $A_{ik}$ can be treated as the random walk probability from the $i$th row index to the $k$th column index and $P_{ij}$ is interpreted as the two-step random walk probability between two row indices $i$ and $j$ via any column index $k$ (with uniform prior over row indices). Besides computational considerations, the choice of which projection to use is also data-dependent. 

\section{Spherical Embedding of Doubly Stochastic Similarity Matrices} \label{sec:sedssm}
When the input similarity matrix is doubly stochastic, we find that s-SNE often embeds the data points around a sphere in the low-dimensional space. The phenomenon is illustrated in Figure \ref{tsne_dsm}, where we generated a 2000$\times$2000 similarity matrix with uniform distribution and visualize it by t-SNE. We can see from the left subfigure that the embedding is close to a ball. In contrast, if the matrix is doubly stochastically normalized (by using the Sinkhorn-Knopp method), the resulting embedded points approximately lie around a circle. The same phenomenon also holds for 3D visualizations.

\begin{figure}[t]
	\centering 
	\includegraphics[width=0.23\textwidth]{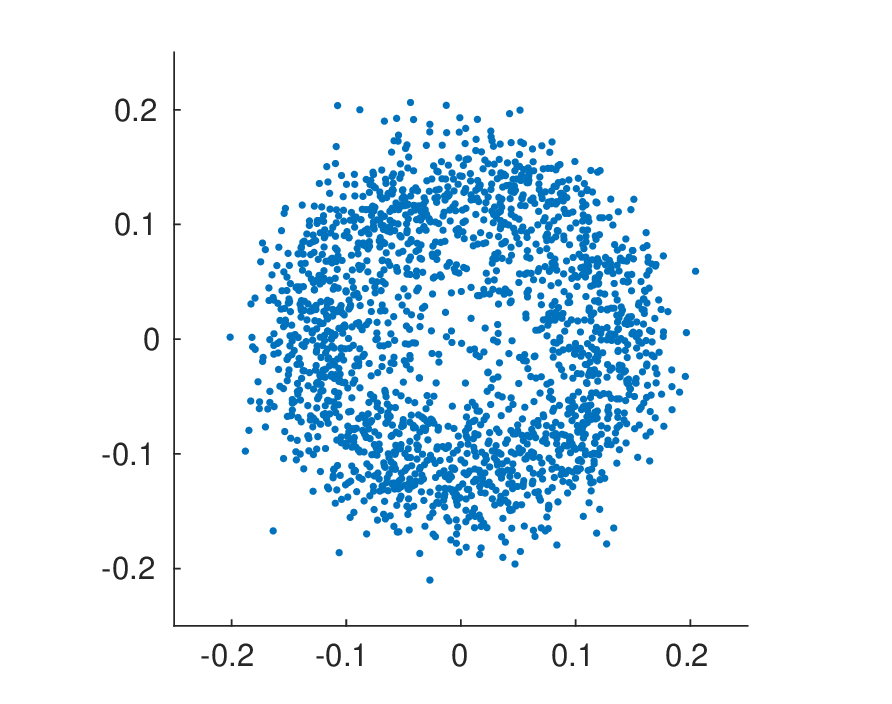}
	\includegraphics[width=0.23\textwidth]{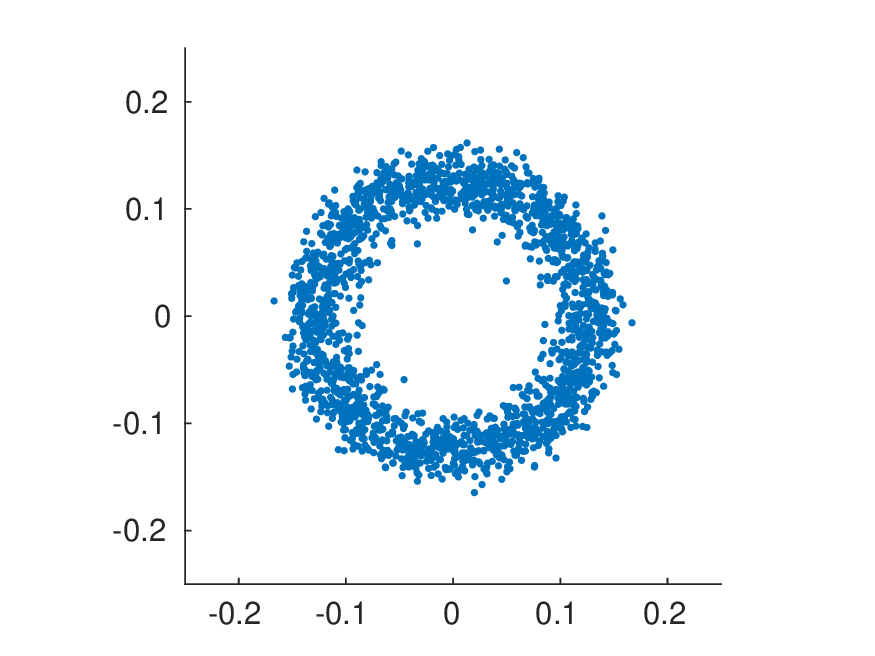}
	\caption{t-SNE visualization of a random uniformly distributed matrix (left) and a random doubly stochastic matrix (right).}\label{tsne_dsm}
\end{figure}

We provide a theoretical analysis of this phenomenon. If $P$ is doubly stochastic, then $Q$ is often approximately doubly stochastic (up to a constant factor) because it approximates $P$ by the KL-divergence. That is, $\sum_jQ_{ij}$ is approximately the same for all $i$. For example, in Figure \ref{tsne_dsm} (right), $\sum_jQ_{ij}$ mainly distribute around a constant (with mean $0.0005$ and very small standard deviation $1.7\times 10^{-6}$). 
In this case, we show that $\sum_j\|y_i-y_j\|^2$ becomes approximately the same for all $i$, bounded by constants, in Proposition \ref{prop:bnd}. Furthermore, we show that when $\sum_j\|y_i-y_j\|^2$ is exactly the same for all $i$, the embedded points must be on a sphere, in Proposition \ref{prop:sph}. The proofs of the propositions are provided in the supplemental document.

\begin{proposition} 
	\label{prop:bnd}
	If $\sum_jq_{ij}=c$ for $i=1,\dots,n$ and $c>0$, then 
	%	\begin{align*}
	$L\leq\sum_j\|y_i-y_j\|^2\leq U$,
	%	\vspace{-2mm}
	%	\end{align*}
	where 1) for $q_{ij}=\exp(-\|y_i-y_j\|^2)$, $L=n\ln\frac{n}{c}$ and $U=n\ln\frac{n}{c-nb}$, with $b = a + (1-a)m - m^a$, $m = \min_j \exp(-\| y_i - y_j \|^2)$ and $a =\displaystyle  \frac{\ln[\ln (1/m) / (1-m)]}{\ln (1/m)}$;
	2) for $q_{ij}=(1+\| y_i - y_j \|^2)^{-1}$, $L=\frac{n^2}{c}-n$ and $U=\frac{n^2}{c}-n + n(\sqrt{b} - 1)^2$, with $b = 1+\max_j\|y_i - y_j\|^2$.
\end{proposition}

\begin{proposition} 
	\label{prop:sph}
	If $\sum_j \| y_i - y_j \|^2 = c$ for $i=1,\dots,n$, $c>0$ and $\sum_i y_i = 0$, then
	%	\begin{align*} 
	$	\|y_1\|^2 = \|y_2\|^2 = \dots = \|y_n\|^2$.
	%	\end{align*}
\end{proposition}

Since the embedding is often nearly spherical for doubly stochastic similarity matrices, it is more suitable to replace the 2D Euclidean embedding space with spheres in 3D space. The resulting layout can be encoded with $n\times 2+1$ numbers (two angles for each data point plus the common radius). Therefore the embedding is still intrinsically two-dimensional.

The spherical geometry itself brings other benefits for visualization. First, the embedding in the Euclidean space has a global center in the middle, while on spheres there is no such global center. Therefore a spherical visualization is free of the ``crowding-in-the-center'' problem. Every point on the sphere can be a local center, which provides fish-eye views for navigation and for examining patterns beyond centrality. Second, the attractive and repulsive forces can be transmitted in a cyclic manner, which helps in discovering macro patterns such inter-cluster similarities.

We thus formulate our learning objective as follows:
\begin{align}
\minimize_{Y\in\bbS} ~&~ \calJ(Y) = \calD_\text{KL}(P||Q),
\end{align}
where $\calJ(Y)$ is an SNE objective function with $P$ doubly stochastic, $Q$ defined in Section \ref{sec:ne} and
\begin{align}
\bbS=\left \{Y~\Big|~Y\in\bbR^{n\times3};\|y_1\| = \dots = \|y_n\|;~\sum_iy_i=0\right\}.
\end{align}
We call the new method DOubly Stochastic Neighbor Embedding on Spheres (DOSNES).
Note that $\bbS$ includes all centered spheres in the three-dimensional space, not only the unit sphere.

We employ a projection step after each SNE update step to enforce the sphere constraint. The DOSNES algorithm steps are summarized as follows: 
%\newpage
\begin{enumerate}
	\item Normalize $P$ to be doubly stochastic.
	\item Repeat until convergence
	\begin{enumerate}
		\item $\widetilde{Y}\leftarrow$OneStepUpdateSNE($P$, $Y$),
		\item $Y\leftarrow\arg\min_{Z\in\bbS}\|Z-\widetilde{Y}\|$.
		\label{step:proj}
	\end{enumerate}
\end{enumerate}
The projection step \ref{step:proj} is performed by implicitly switching $\widetilde{Y}=[\tilde{y}_1,\dots,\tilde{y}_n]^T$ to the spherical coordinate system, taking the mean radius, and switching back to Cartesian coordinates. This is implemented as:
%for $i=1,\dots,n$
%\begin{align}
%y_i \leftarrow  \frac{\tilde{y}_i}{\|\tilde{y}_i\|} \cdot \left(\frac{1}{n}\sum_j \|\tilde{y}_j\|\right),
%\end{align}
%where $\tilde{y}_i = \tilde{y}_i-\frac{1}{n}\sum_j\tilde{y}_j$.
%
For $i=1,\dots,n$
\begin{align}
\tilde{y}_i &\leftarrow \tilde{y}_i-\frac{1}{n}\sum_j\tilde{y}_j, \\
y_i &\leftarrow  \frac{\tilde{y}_i}{\|\tilde{y}_i\|} \cdot \left(\frac{1}{n}\sum_j \|\tilde{y}_j\|\right).
\end{align}
The iterations converge to a stationary point with suitable learning step sizes \citep[see e.g.,][Section 5]{iusem2003pgconvergence}.

We do not take gradient steps directly with respect to the latitude and longitude parameters, because such gradients give no information for learning the sphere radius and require expensive trigonometric functions.

\section{Experiments}
\label{sec:exp}
We developed a browser-based software for displaying and navigating the DOSNES results. The software and its demos can be found in the project website\footnote{\url{http://yaolubrain.github.io/dosnes/}}. In the paper we present the 2D projected views of the spheres.

We compare our proposed method DOSNES with two- and three-dimensional t-SNE\footnote{\url{https://lvdmaaten.github.io/tsne/}} as well as non-metric MDS\footnote{We used the \texttt{isoMDS()} function in the MASS R package.} in Euclidean embedding space \citep{van2008visualizing} to verify the effectiveness of using doubly stochastic similarities and the sphere constraint. 

The compared methods were tested on three real-world datasets from different domains:

1) \texttt{NIPS}\footnote{\url{https://papers.nips.cc/}}: the proceedings of NIPS conferences (1987-2015) which contains 5,993 papers and their associated 6,621 authors. We used the largest connected component in the co-author graph with 5,300 papers and 5,422 authors. The (non-normalized) similarity matrix is from the co-author graph, i.e. $BB^T$ where $B$ is the author-paper co-occurrence matrix.

2) \texttt{WorldTrade}\footnote{\url{http://vlado.fmf.uni-lj.si/pub/networks/data/esna/metalWT.htm}}: trade network of metal manufactures among 80 countries in 1994. Each edge represents the total trade amount (imports and exports) between two countries.

3) \texttt{MIREX}\footnote{\url{http://www.music-ir.org/mirex/wiki/2007}}: the dataset is from the the Third Music
Information Retrieval Evaluation eXchange (MIREX
2007). It is a
network of 3090 songs in 10 music genre classes. The
weighted edges are human judgment on how similar two songs are.

MDS requires a distance matrix as input. Given a similarity matrix $S$, we first normalize $\widetilde{S}_{ij}=S_{ij}/\max(S)$. Treating $\widetilde{S}_{ij}'s$ as cosine similarities, we obtain the cosine distances by $D_{ij}=1-\widetilde{S}_{ij}$. Next we calculate the shortest graph distances between all nodes and feed them to MDS.

The \texttt{NIPS} co-author graph is visualized in Figure \ref{fig:nips}. The node degrees of the graph are highly uneven, where many authors have only one paper while the the most productive author has 93 papers. In Figure \ref{fig:nips} (a) and (b), we can see both 2D and 3D t-SNE caused the most productive \texttt{NIPS} authors crowded in the center. This is undesirable because these authors actually do not often co-author \texttt{NIPS} papers. For example, \texttt{Hinton\_G} has no co-authored paper with \texttt{Sch\"olkopf\_B} but they are very close in the t-SNE layout. 
A similar crowding problem is observed in the MDS visualizations.
In Figure \ref{fig:nips} (e) and (f), DOSNES resolves neatly the crowding problem, by normalizing the similarity matrix with our method in Section \ref{sec:dssm} and visualizing the authors with spherical layout. The productive \texttt{NIPS} authors are now more evenly distributed. For example, \texttt{Hinton\_G} becomes more distant to \texttt{Sch\"olkopf\_B}. Meanwhile, retrieval around the most established authors reveals accurate co-authorship. For example, \texttt{Revow\_M}, \texttt{Nair\_V} and \texttt{Brown\_A} are close to \texttt{Hinton\_G} because all their \texttt{NIPS} papers are co-authored with \texttt{Hinton\_G}. See our online demo\footnote{\url{http://yaolubrain.github.io/dosnes/demo/nips/}} for more details.

\newcommand{\figwidth}{5.7cm}
\begin{figure*}[t]
	\centering
	\subcaptionbox{t-SNE 2D}{
		\fbox{\includegraphics[height=\figwidth,width=\figwidth]{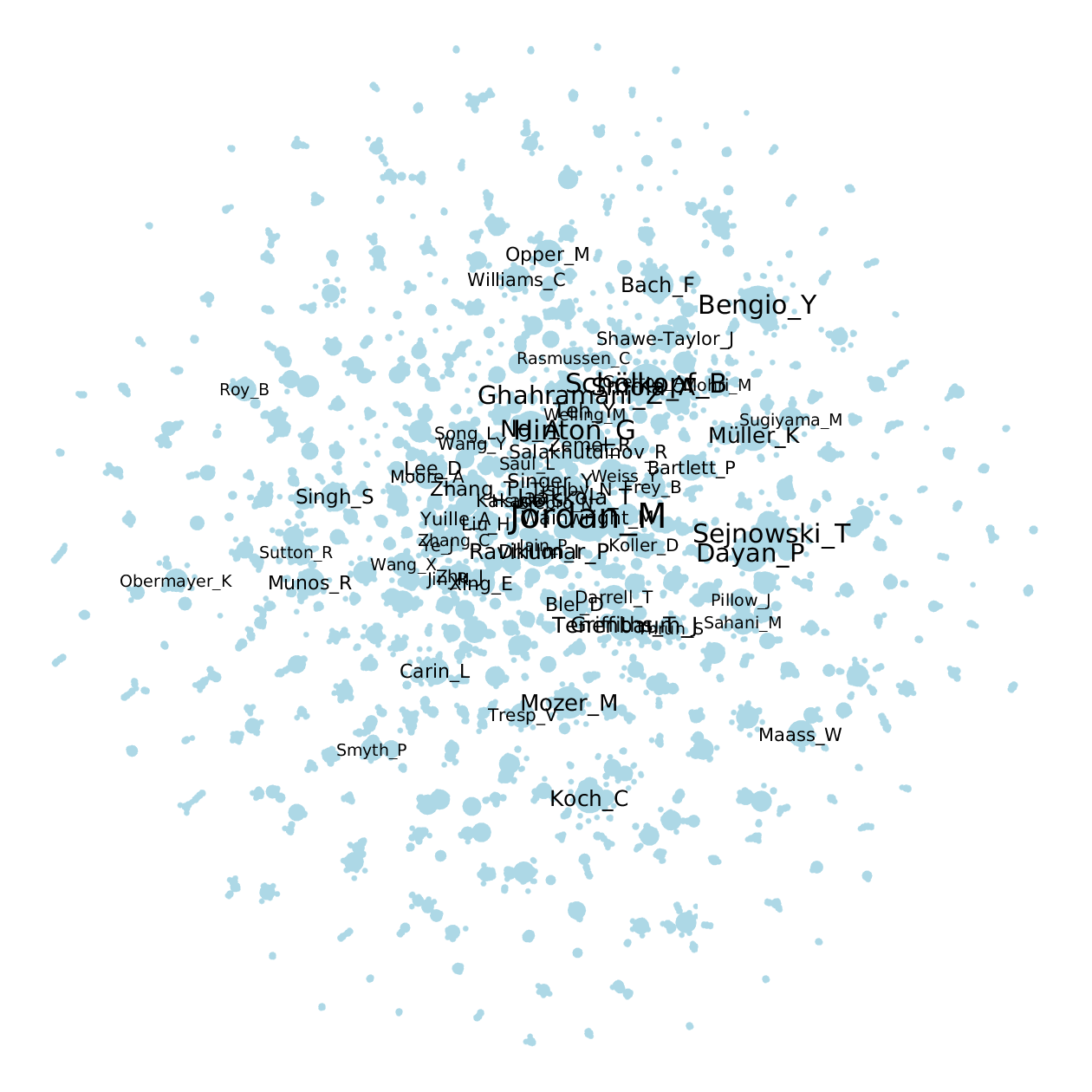}}}
	\subcaptionbox{t-SNE 3D}{
		\fbox{\includegraphics[height=\figwidth,width=\figwidth]{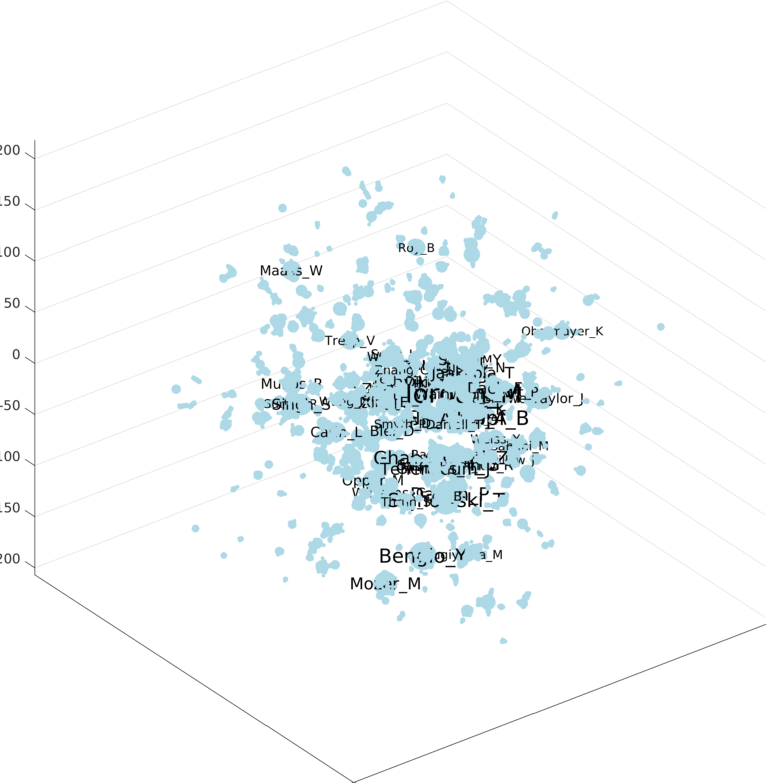}}}  
	
	\vspace{\baselineskip}
	
	\subcaptionbox{MDS 2D}{
		\fbox{\includegraphics[height=\figwidth,width=\figwidth]{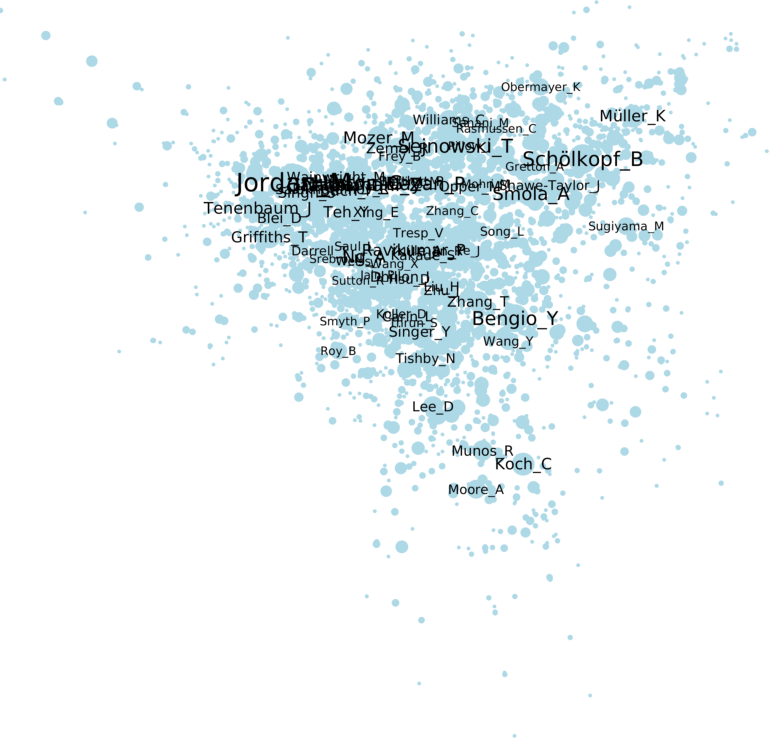}}}
	\subcaptionbox{MDS 3D}{
		\fbox{\includegraphics[height=\figwidth,width=\figwidth]{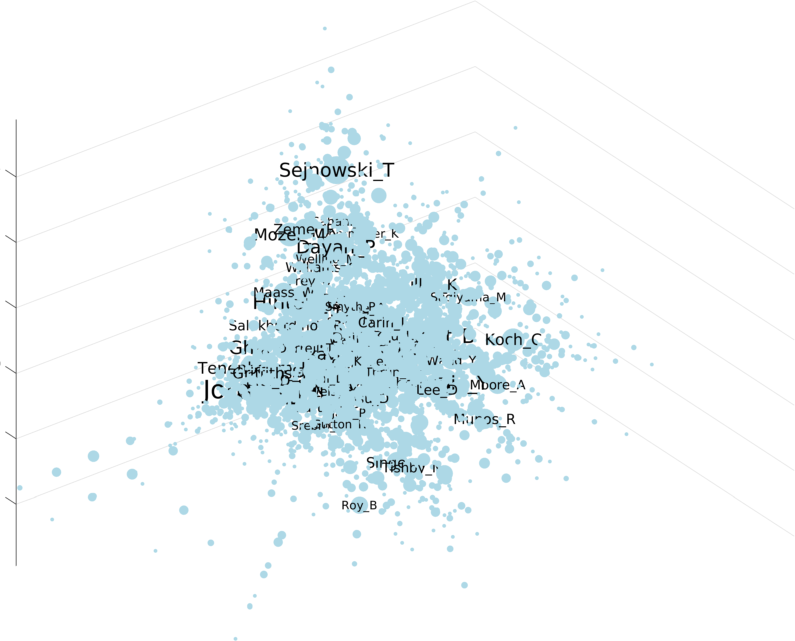}}}  
	
	\vspace{\baselineskip}
	
	\subcaptionbox{DOSNES (viewpoint 1)}{\includegraphics[width=\figwidth]{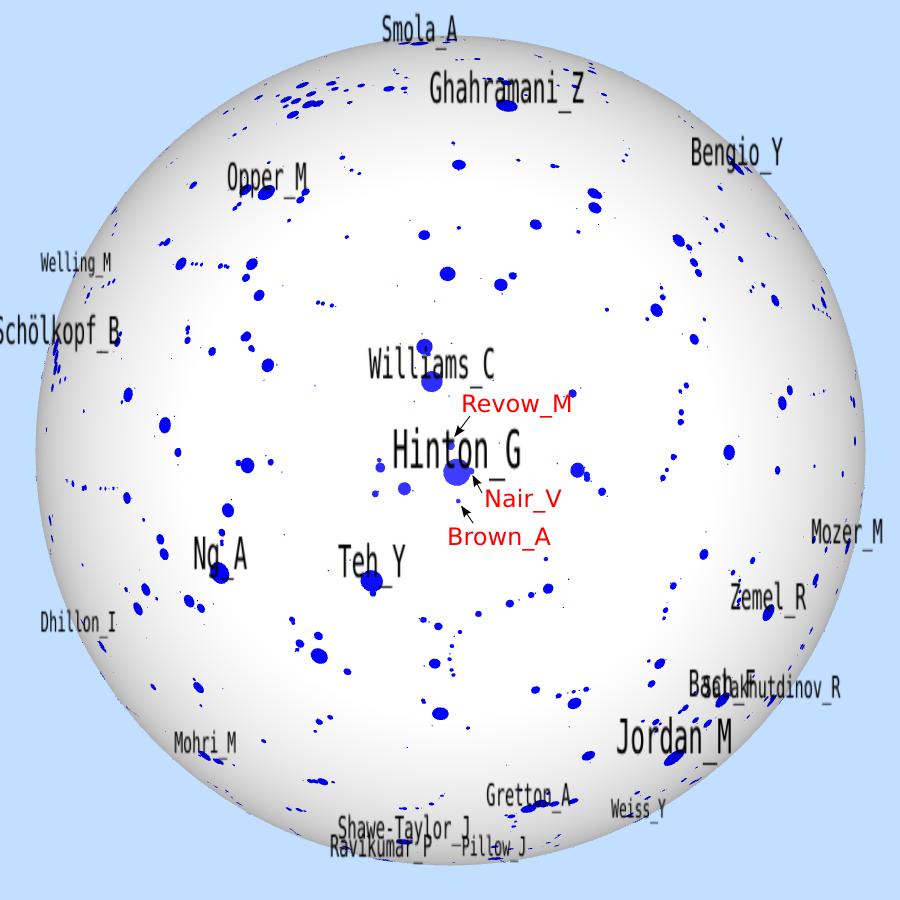}}
	\subcaptionbox{DOSNES (viewpoint 2)}{\includegraphics[width=\figwidth]{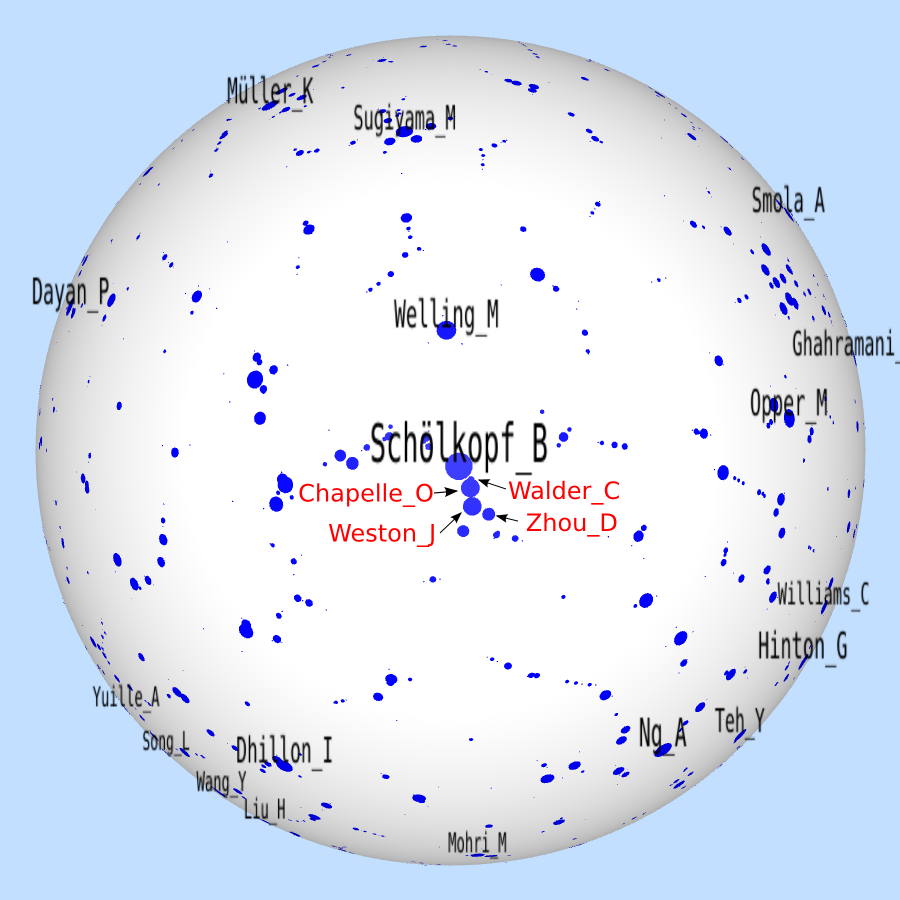}}
	\caption{Visualizations of the \texttt{NIPS} dataset.}
	\label{fig:nips}
\end{figure*}

The visualizations of the \texttt{WorldTrade} graph are given in Figure \ref{fig:worldtrade}. In this graph, some countries such as \texttt{United States} and \texttt{Germany} have more total trade amount than many others. In Figure \ref{fig:worldtrade} (a) to (d), we can see both 2D and 3D t-SNE, as well as the MDS visualizations, caused these countries crowded in the center. In contrast, DOSNES places the countries more evenly. In Figure \ref{fig:worldtrade} (e) and (f), we can see on the sphere many meaningful clusters (e.g., \texttt{Europe} and \texttt{Asia}) which well match the geography even though we did not use such information in the training. See our demo globe\footnote{\url{http://yaolubrain.github.io/dosnes/demo/worldtrade/}} for other viewpoints.

\begin{figure*}[t]
	\centering
	\subcaptionbox{t-SNE 2D}{
		\fbox{\includegraphics[height=\figwidth,width=\figwidth]{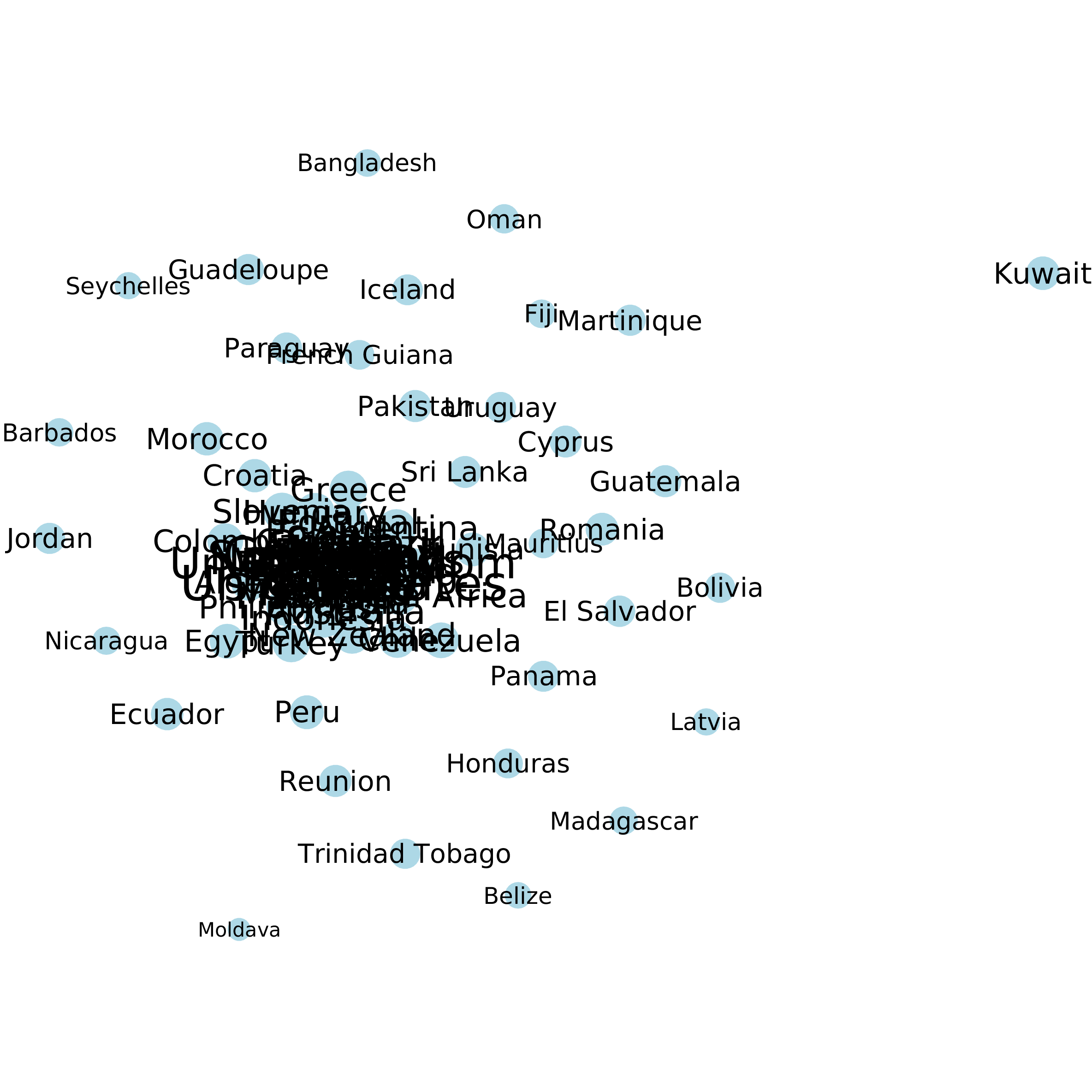}}}
	\subcaptionbox{t-SNE 3D}{
		\fbox{\includegraphics[height=\figwidth,width=\figwidth]{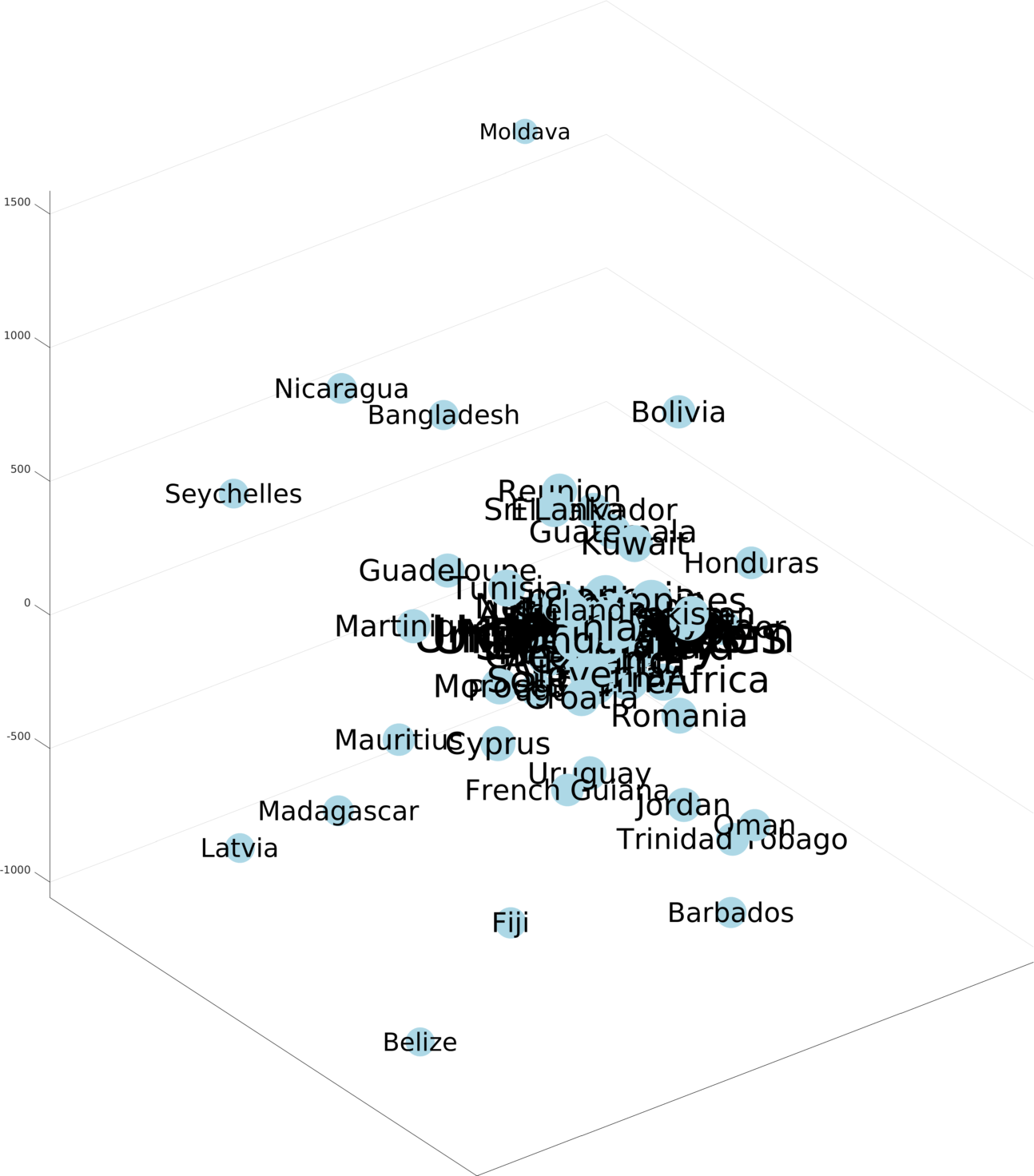}}}  
	
	\vspace{\baselineskip}
	
	\subcaptionbox{MDS 2D}{
		\fbox{\includegraphics[height=\figwidth,width=\figwidth]{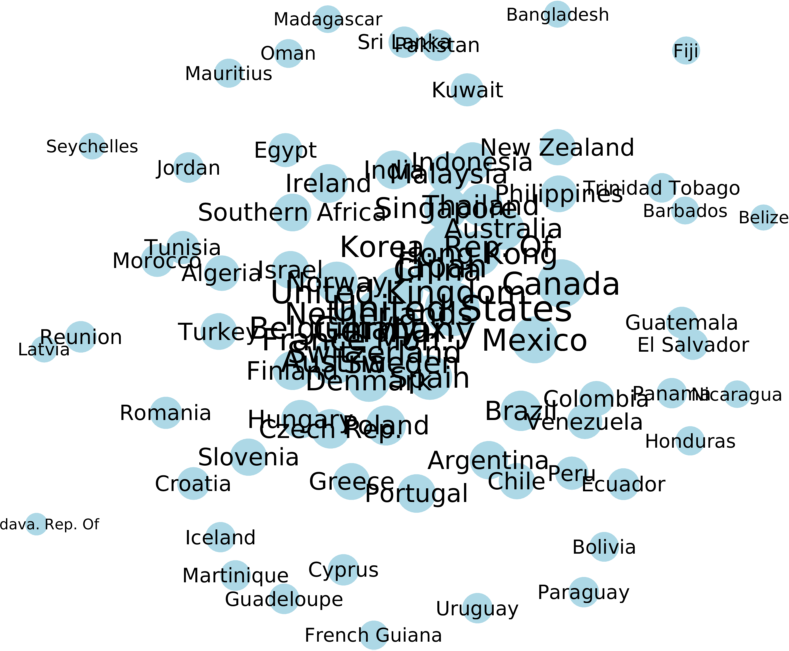}}}
	\subcaptionbox{MDS 3D}{
		\fbox{\includegraphics[height=\figwidth,width=\figwidth]{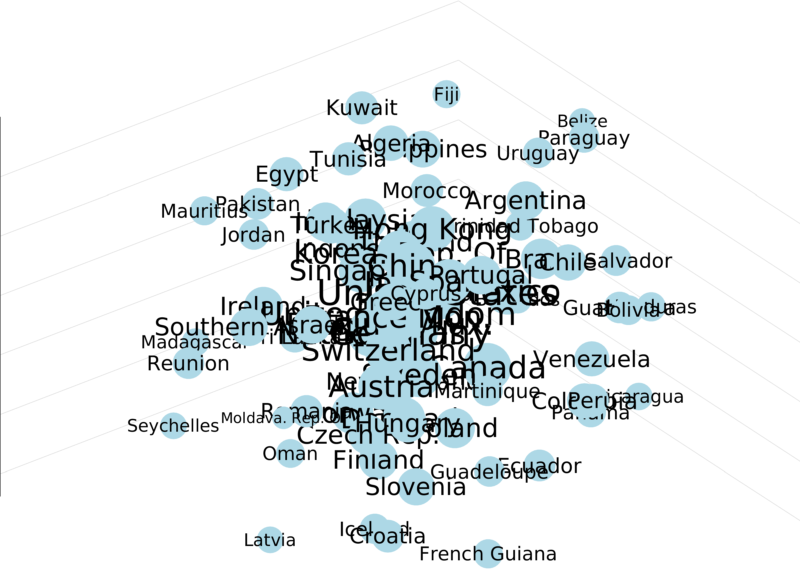}}}  
	
	\vspace{\baselineskip}
	
	\subcaptionbox{DOSNES (viewpoint 1)}{
		\includegraphics[width=\figwidth]{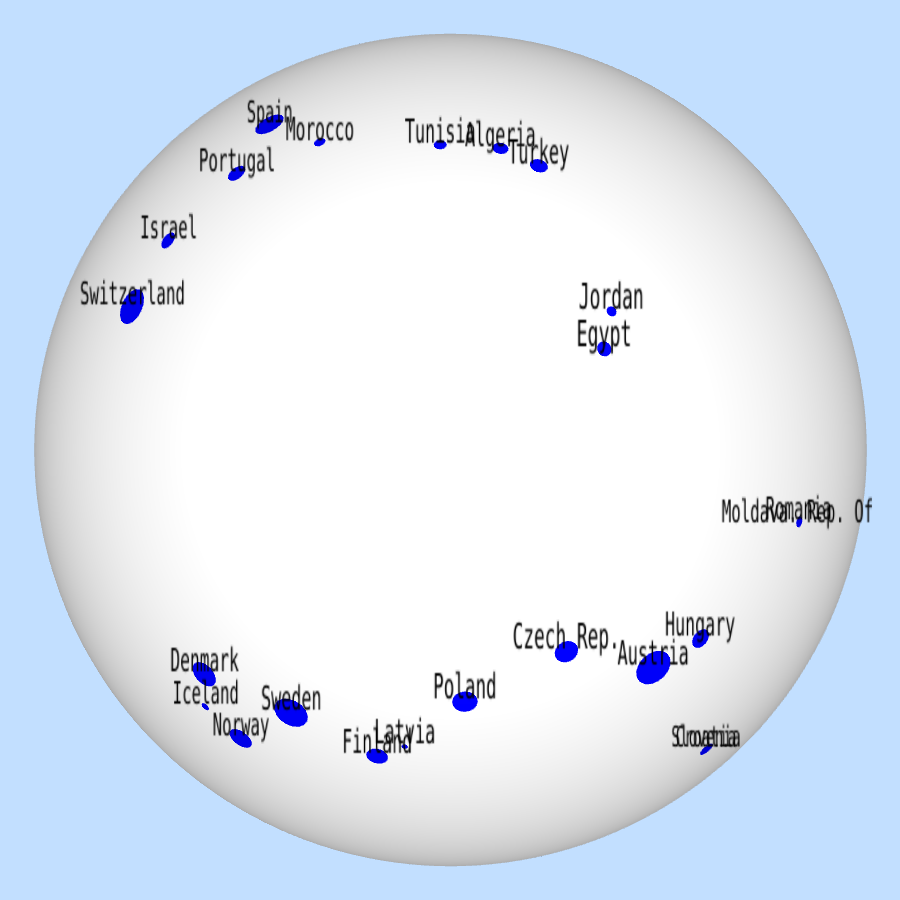}}
	\subcaptionbox{DOSNES (viewpoint 2)}{
		\includegraphics[width=\figwidth]{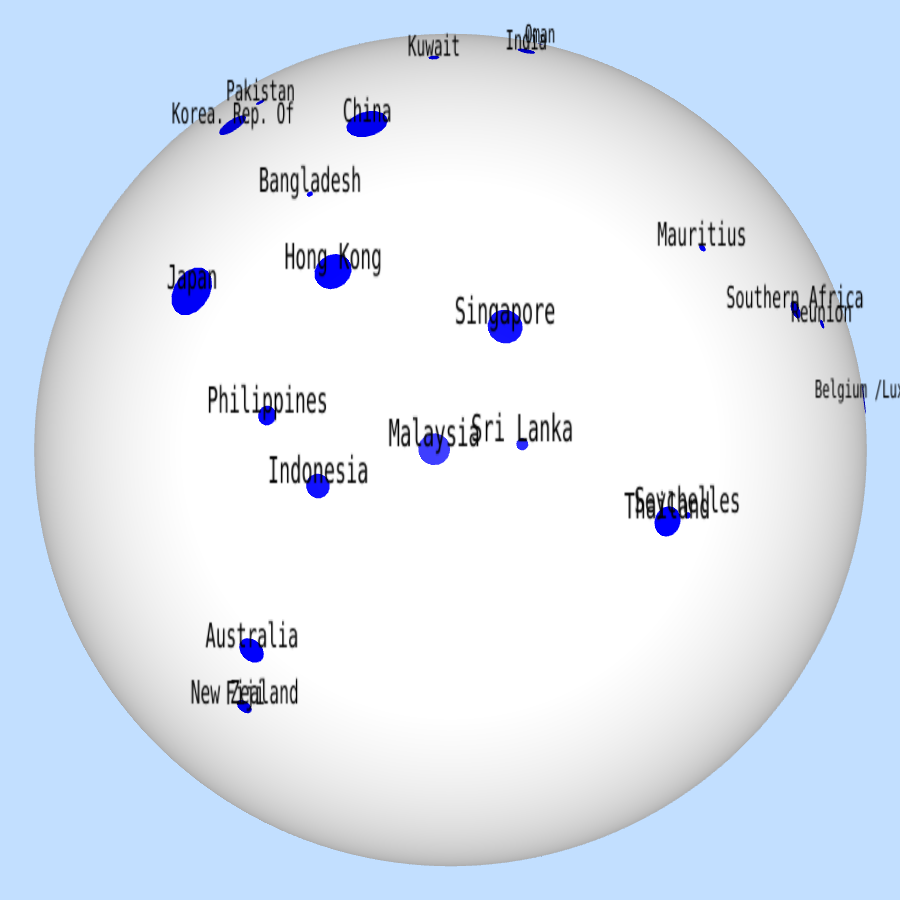}}
	\caption{Visualizations of the \texttt{WorldTrade} dataset.} 
	\label{fig:worldtrade}
\end{figure*}

Figure \ref{fig:mirex} gives the visualizations of the \texttt{MIREX} dataset. In the panels (a) and (b), we can see that t-SNE caused over 90 percent of songs crowded in the center. A similar crowding problem appears in the MDS visualizations (panels c and d). In contrast, DOSNES performs much better in terms of separating the song genres and their subgroups, as in Figure \ref{fig:mirex} (e) and (f).

\begin{figure*}[t]
	\centering
	\subcaptionbox{t-SNE 2D}{
		\fbox{\includegraphics[height=\figwidth, width=\figwidth]{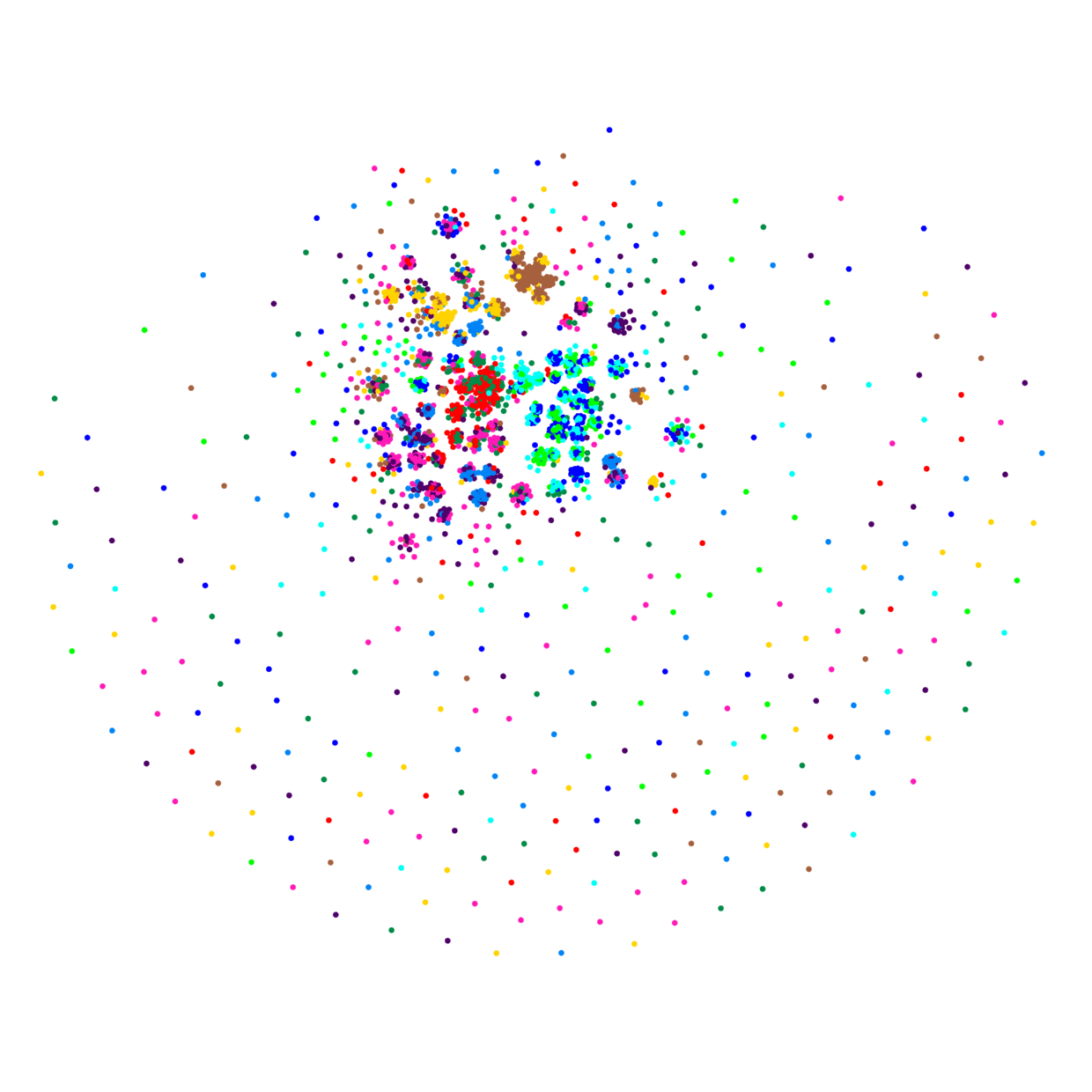}}}
	\subcaptionbox{t-SNE 3D}{
		\fbox{\includegraphics[height=\figwidth,width=\figwidth]{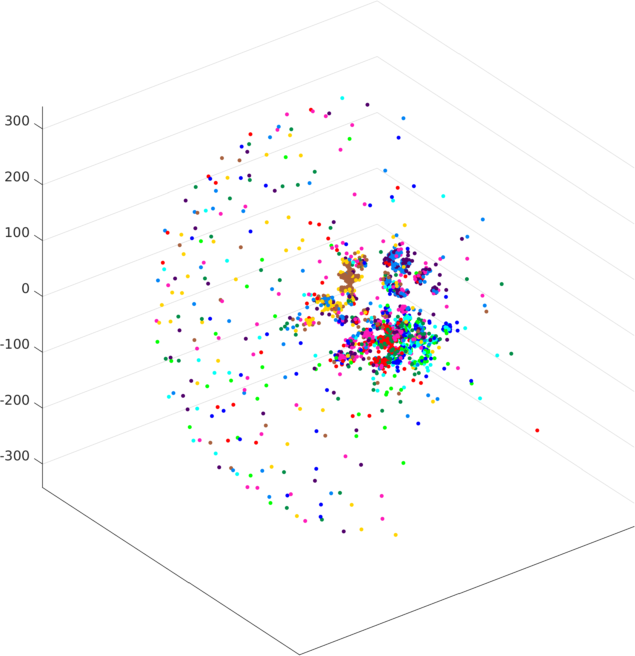}}}
	
	\vspace{\baselineskip}  
	
	\subcaptionbox{MDS 2D}{
		\fbox{\includegraphics[height=\figwidth, width=\figwidth]{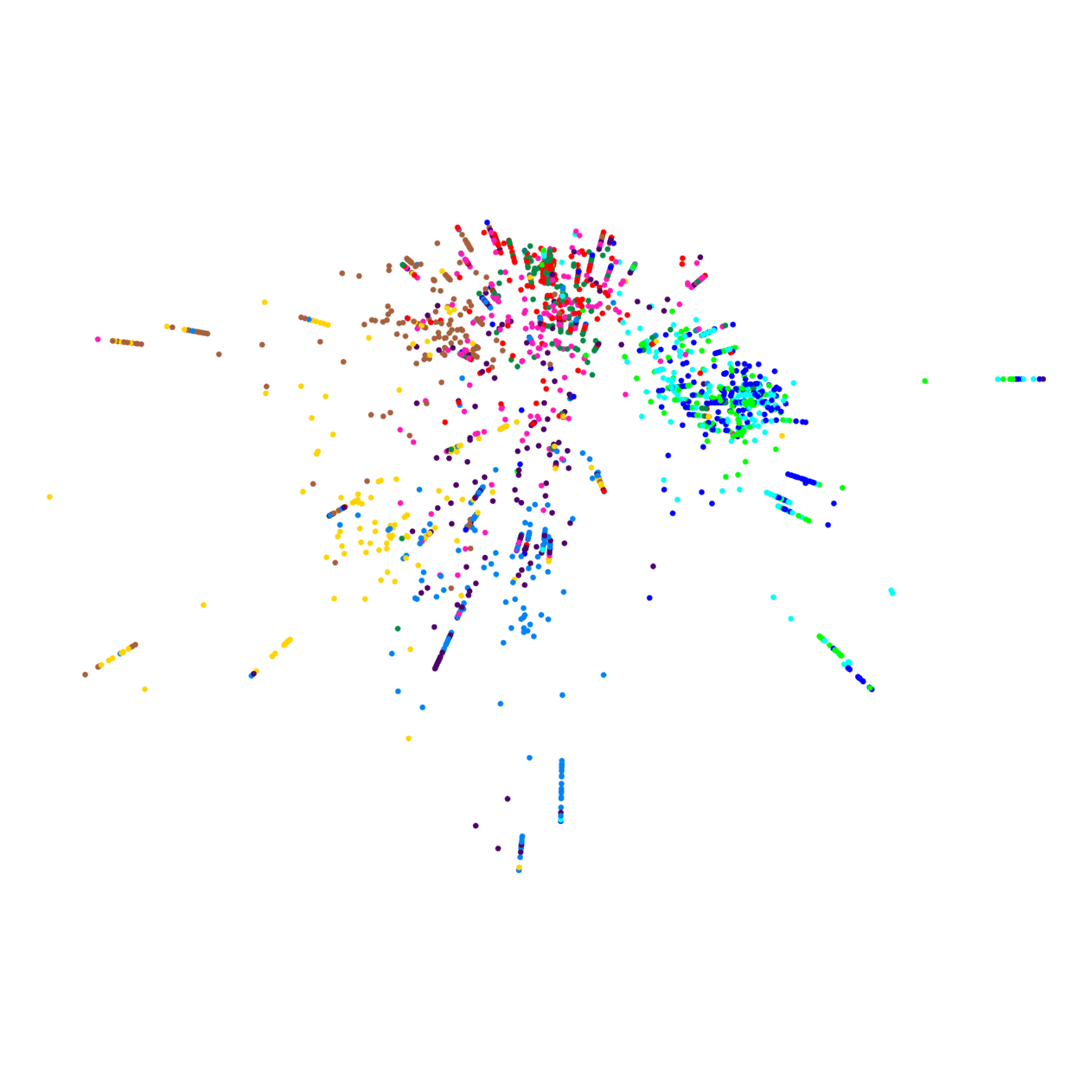}}}
	\subcaptionbox{MDS 3D}{
		\fbox{\includegraphics[height=\figwidth,width=\figwidth]{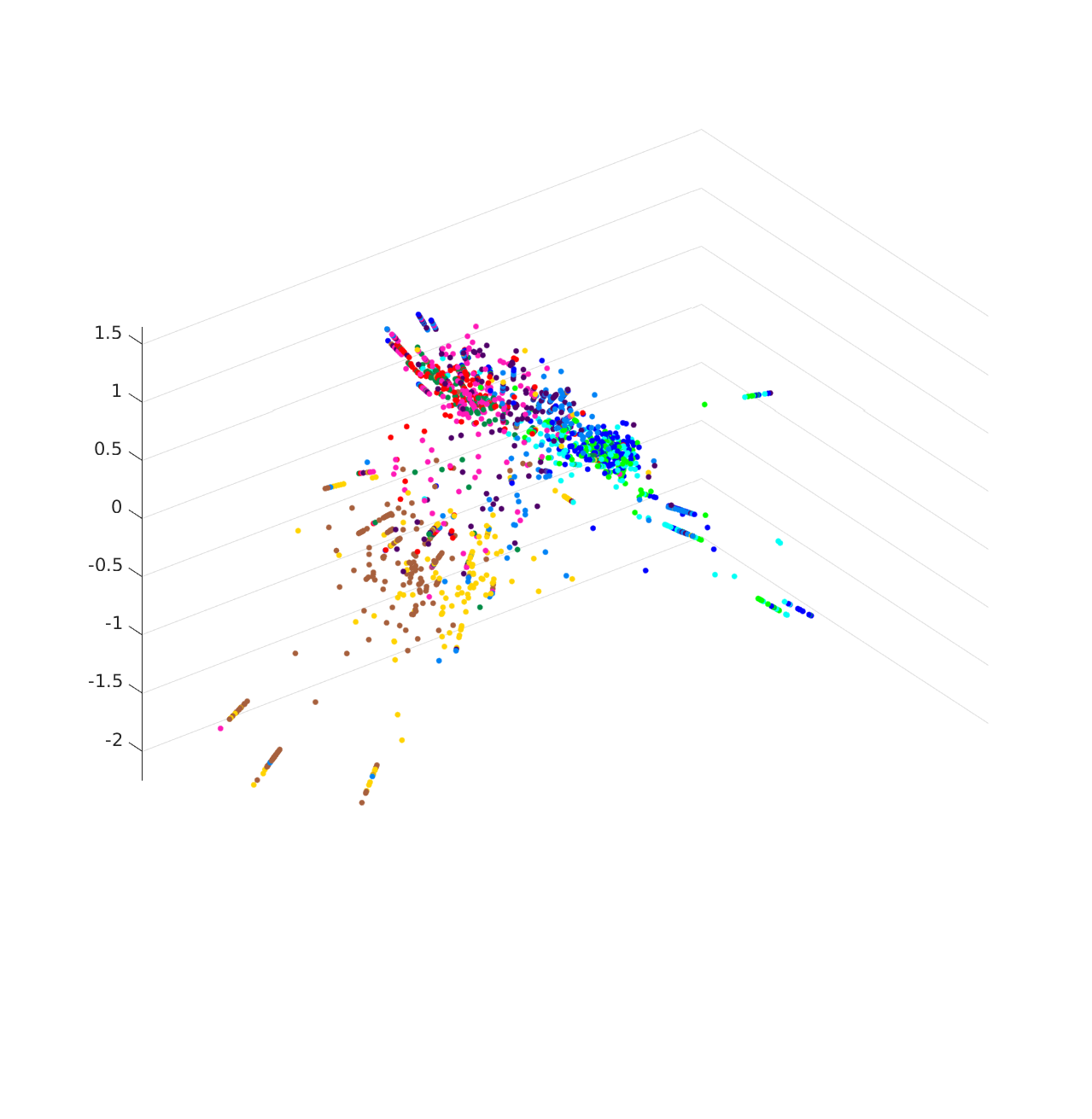}}}
	
	\vspace{\baselineskip}  
	
	\subcaptionbox{DOSNES (viewpoint 1)}{  
		\includegraphics[width=\figwidth]{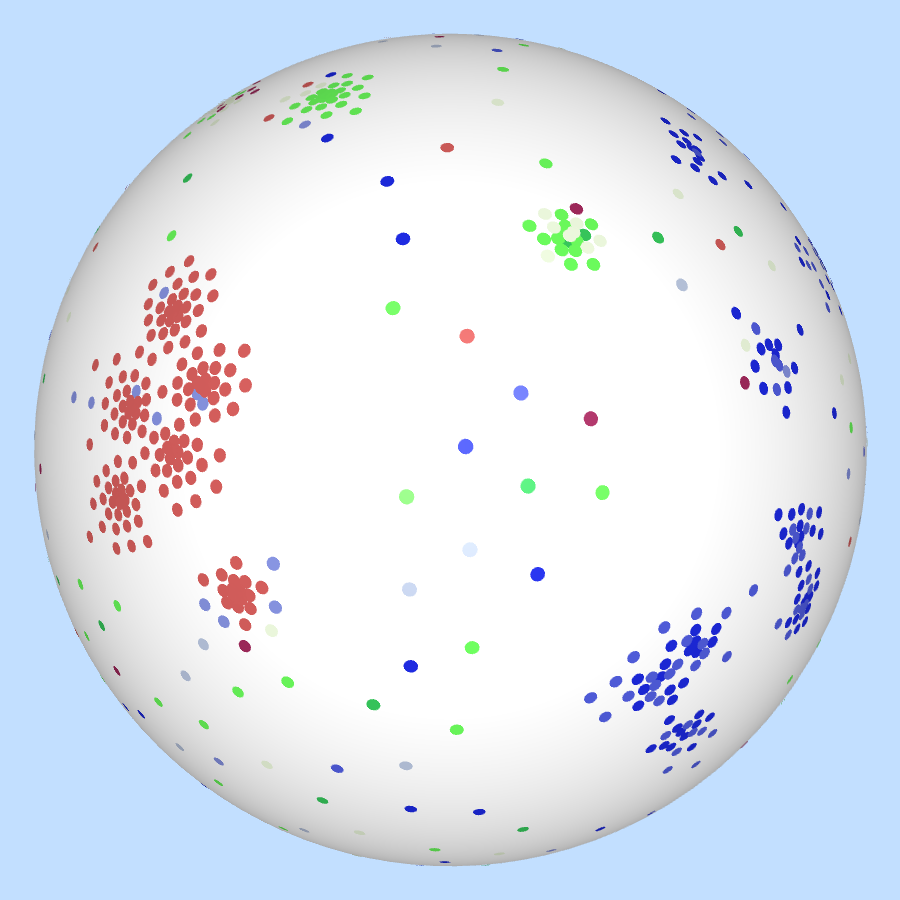}}
	\subcaptionbox{DOSNES (viewpoint 2)}{
		\includegraphics[width=\figwidth]{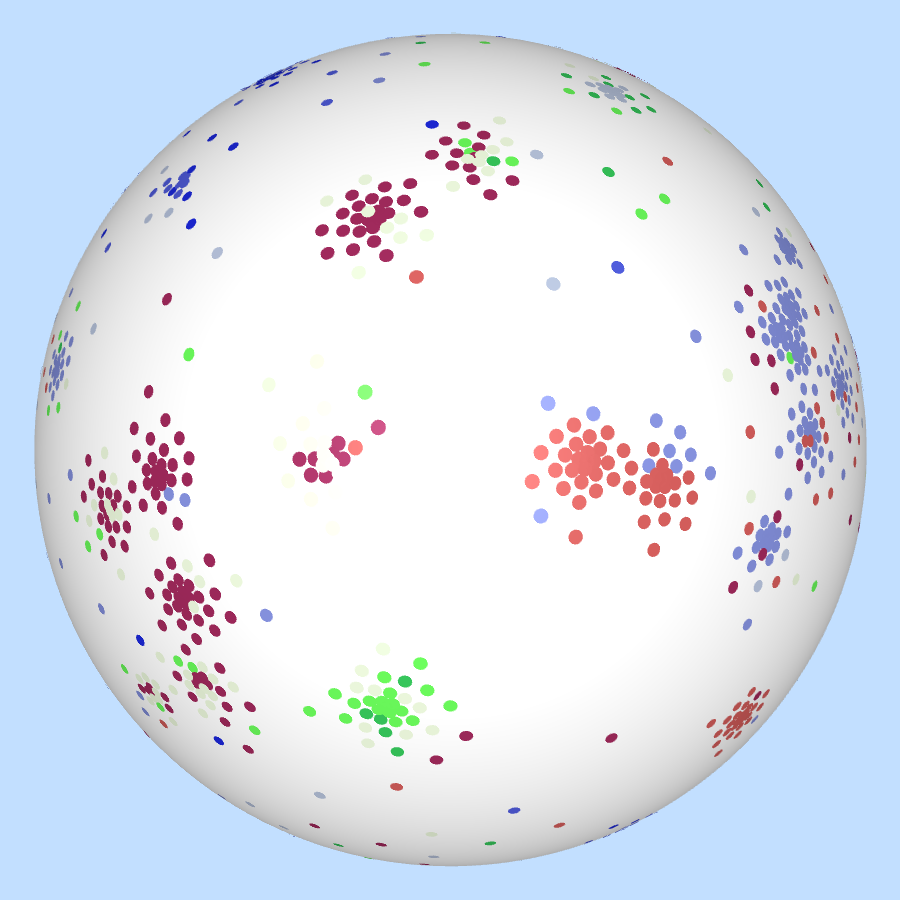}}  
	\caption{Visualizations of the \texttt{MIREX} dataset.}
	\label{fig:mirex}
\end{figure*}

\section{Conclusions}
\label{sec:conclusion}
We have presented a new visualization method for high-dimensional and graph data. The proposed DOSNES method is based on the Stochastic Neighbor Embedding principle but with two key improvements: we normalize the input similarity matrix to be doubly stochastic and replace the 2D Euclidean embedding space with spheres in 3D space. Empirical results show that our method significantly outperforms the state-of-the-art approach t-SNE in terms of resolving the crowding problem and preserving intrinsic similarities.

We will perform a more thorough theoretical study on the connection between doubly stochastic similarity matrix and spherical embedding. There could be many possibilities to improve the user interface of spherical visualization, for example, using spherical screens to further facilitate the data navigation.

\appendix

\section{Proofs of Propositions in the Paper}

\begin{proposition} 
	If 
	\begin{align}
	\sum_j \exp(-\| y_i - y_j \|^2) = c
	\end{align}
	for $i = 1...n$ and some $c > 0$, where $y_i \in \bbR^d$, then \\
	\begin{align}
	n\ln\frac{n}{c} \leq \sum_j \| y_i - y_j \|^2 \leq n\ln\frac{n}{c-nb} 
	\end{align}
	where $b = a + (1-a)m - m^a$, $m = \min_j \exp(-\| y_i - y_j \|^2)$ and $a = \frac{\ln[ \ln (1/m) / (1-m)]}{\ln (1/m)}$.
\end{proposition}

\begin{proof}
	Let 
	\begin{align}
	A &= \frac{1}{n}\sum_j \exp(-\| y_i - y_j \|^2) = \frac{c}{n} \quad   (\text{arithmetic mean})\\
	G &=  [\exp(-\sum_j\| y_i - y_j \|^2)]^{1/n} \quad (\text{geometric mean})\\
	\end{align}
	For lower bound, we have
	\begin{align}
	G &\leq A \\
	[\exp(-\sum_j\| y_i - y_j \|^2)]^{1/n} &\leq \frac{c}{n}  \\
	\exp(-\sum_j\| y_i - y_j \|^2) &\leq (\frac{c}{n})^{n}  \\
	-\sum_j\| y_i - y_j \|^2 &\leq n\ln  \frac{c}{n} \\
	n\ln \frac{n}{c}  &\leq  \sum_j\| y_i - y_j \|^2 \\
	\end{align}
	For upper bound, by Tung Theorem \citep{tung1975lower}, we have 
	\begin{align}
	A - G \leq am + (1-a)M - m^aM^{1-a}
	\end{align}
	where $m = \min_j \exp(-\| y_i - y_j \|^2)$, $M = \max_j \exp(-\| y_i - y_j \|^2)$ and 
	\begin{align}
	a = \frac{\ln[M/(M-m) \ln (M/m)]}{\ln (M/m)}
	\end{align}
	Since $\max_j\exp(-\| y_i - y_j \|^2)=1$, we have
	\begin{align}
	A - G \leq am + (1-a) - m^a 
	\end{align}
	and 
	\begin{align}
	a = \frac{\ln[1/(1-m) \ln (1/m)]}{\ln (1/m)}
	\end{align}
	Let $b = am + (1-a) - m^a $, we have
	\begin{align}
	G &\geq A - b \\
	[\exp(-\sum_j\| y_i - y_j \|^2)]^{1/n} &\geq \frac{c}{n} - b \\
	n\ln\frac{n}{c-nb}  &\geq \sum_j\| y_i - y_j \|^2
	\end{align}
\end{proof}

\begin{proposition} 
	If 
	\begin{align}
	\sum_j (1+\| y_i - y_j \|^2)^{-1} = c
	\end{align}
	for $i = 1...n$ and some $c > 0$, where $y_i \in \bbR^d$, then \\
	\begin{align}
	\frac{n^2}{c}-n \leq \sum_j \| y_i - y_j \|^2 \leq \frac{n^2}{c}-n + n(b^{1/2} - 1)^2
	\end{align}
	where $b = 1+\max_j\|y_i - y_j\|^2$.
\end{proposition}
\begin{proof}
	Let 
	\begin{align}
	A &= \frac{1}{n}\sum_j (1+\| y_i - y_j \|^2) \quad   (\text{arithmetic mean})\\
	H &=  \frac{n}{\sum_j (1+\| y_i - y_j \|^2)^{-1}} \quad (\text{harmonic mean})\\
	\end{align}
	For upper bound, we have 
	\begin{align}
	A &\geq H = \frac{n}{c}\\
	\frac{1}{n}\sum_j \| y_i - y_j \|^2 + 1 &\geq \frac{n}{c} \\
	\sum_j \| y_i - y_j \|^2 &\geq \frac{n^2}{c} - n \\
	\end{align}
	For lower bound, due to Meyer Theorem \citep{meyer1984some}, let $b = 1 + \max_j \| y_i - y_j \|^2$, then we have 
	\begin{align}
	A - H &\leq (b^{1/2} - 1)^2 \\
	\frac{1}{n}\sum_j (1+\| y_i - y_j \|^2) &\leq \frac{n}{c} + (b^{1/2} - 1)^2 \\
	\sum_j \| y_i - y_j \|^2 &\leq \frac{n^2}{c} - n + n(b^{1/2} - 1)^2
	\end{align}
	
\end{proof}

\begin{proposition} 
	If 
	\begin{align}
	\sum_j \| y_i - y_j \|^2 = c
	\end{align}
	for $i = 1...n$ and some $c > 0$, where $y_i \in \bbR^d$ and $\sum_i y_i = 0$, then \\
	\begin{align}
	\|y_1\|^2 = \|y_2\|^2 = ... = \|y_n\|^2.
	\end{align}
\end{proposition}
\begin{proof}
	\begin{align}
	\sum_j \| y_i - y_j \|^2 &= c \\
	\sum_j (\|y_i\|^2 + \|y_j\|^2 - 2 y_i^Ty_j) &= c \\
	n\|y_i\|^2 + \sum_j\|y_j\|^2 - 2 y_i^T\sum_jy_j &= c \\
	n\|y_i\|^2 + \sum_j\|y_j\|^2  &= c \\
	\|y_i\|^2& = (c - \sum_j\|y_j\|^2 ) / n
	\end{align}
	Since $\sum_j\|y_j\|^2$ is independent of $i$, we have 
	\begin{align}
	\|y_1\|^2 = \|y_2\|^2 = ... = \|y_n\|^2.
	\end{align}
\end{proof}

\section{Related work}
Normalizing a matrix to be doubly stochastic has been used to improve
cluster analysis. Zass and Shashua proposed to improve spectral
clustering by replacing the original similarity matrix by its closest
doubly stochastic similarities under $L_1$ or Frobenius
norm \citep{zass2006doubly}.  \citet{wang2012improving} generalized the projection to the family of
Bregman divergences. To our knowledge, DOSNES is the first method that applies doubly stochastic matrices to improve data visualization.

Spherical visualization has appeared earlier in the visualization literature. For example, Spherical Multidimensional Scaling replaces Euclidean embedding space in the classical MDS with the unit sphere \citep{cox1991multidimensional}. Similar replacement was used by \citet{wilson2010spherical,fang2011slle,lunga2013spherical}. \citet{lunga2013spherical} also changed the output similarities with the Exit distribution, although this makes the objective function non-smooth.

Our method has a critical difference from the above approaches: the DOSNES embedding space is not restricted to the unit sphere. This is advantageous in two aspects: 1) our objective function is smooth and there is no gradient overflow problem; 2) DOSNES does not require an explicit scale variable for $Y$ or kernel bandwidth variable for $Q$, which is difficult to optimize. The sphere radius in DOSNES is implicitly adapted during optimization \citep[similar to the scale adaptation in SNE;][]{hinton2002stochastic}. Moreover, the embedding space in our method is continuous, unlike sperical SOM \citep{boudjemai2003sphericalsom} or spherical GTM \citep{bishop1998gtm} that use discrete layout of cells.

The DOSNES spherical layout roots in the use of doubly stochastic similarity matrix and aims at solving the crowding problem in SNE, not only to eliminate the border effect. DOSNES does not require that the input high-dimensional data must be on spheres. This is also quite different from the methods that embed high-dimensional spherical data to low-dimensional spherical one \citep{lunga2013spherical,wang2016vmf}. 
%

%The sphere constraint in DOSNES is only on the low-dimensional output space.

Compared with hyperbolic visualizations \citep[see e.g.,][]{hyperbolictree,Munzner1995hyperbolic}, the DOSNES display and navigation are more natural for most viewers without comprehensive knowledge of the transformation models such as Klein or Poincar\'e.

In computational and graphical statistics, \citet{wagaman2009discovering} proposed a bootstrap approach to constructing the neighborhood graph used by Isomap \citep{tenenbaum2000global}. Their method, however, is not extensible to existing graph or network data. In another work, \citet{faraway2012backscoring} discussed the backscoring mapping from the embedding to dissimilarities in the input space, though their work is restricted to classical multidimensional scaling and thus give no insight to handle imbalanced affinity matrix.  

\section{Visualizations by Gaussian s-SNE}

In the paper we show the visualizations by using s-SNE with the Cauchy kernel (i.e.~t-SNE). Here we provide the results by using s-SNE with Gaussian kernel in Figure \ref{fig:ssne}. We can see that the Gaussian s-SNE also suffer from the ``crowding-in-the-center'' problem, where many data nodes (points) with large degrees are crowded in the middle of the display. Same as t-SNE, The failure holds for both 2D and 3D Gaussian s-SNE. This further confirms the existence of ``crowding problem'' in s-SNE if the similairty matrix is not normalized.

\renewcommand{\figwidth}{0.37\textwidth}

\begin{figure*}[t]
	\centering
	\subcaptionbox{s-SNE 2D for \texttt{NIPS}}{
		\fbox{\includegraphics[height=\figwidth, width=\figwidth]{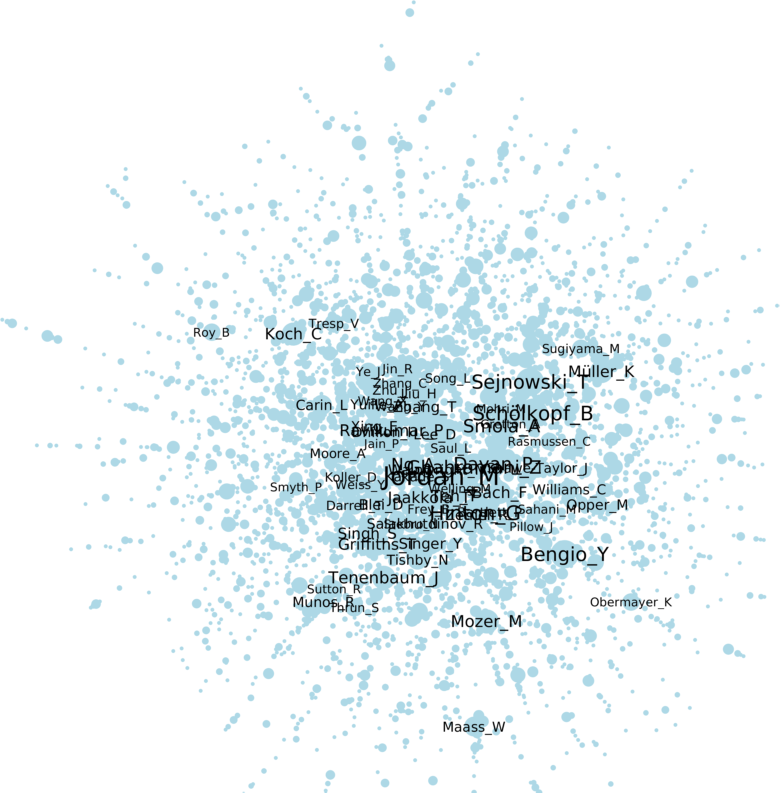}}}
	\subcaptionbox{s-SNE 3D for \texttt{NIPS}}{
		\fbox{\includegraphics[height=\figwidth,width=\figwidth]{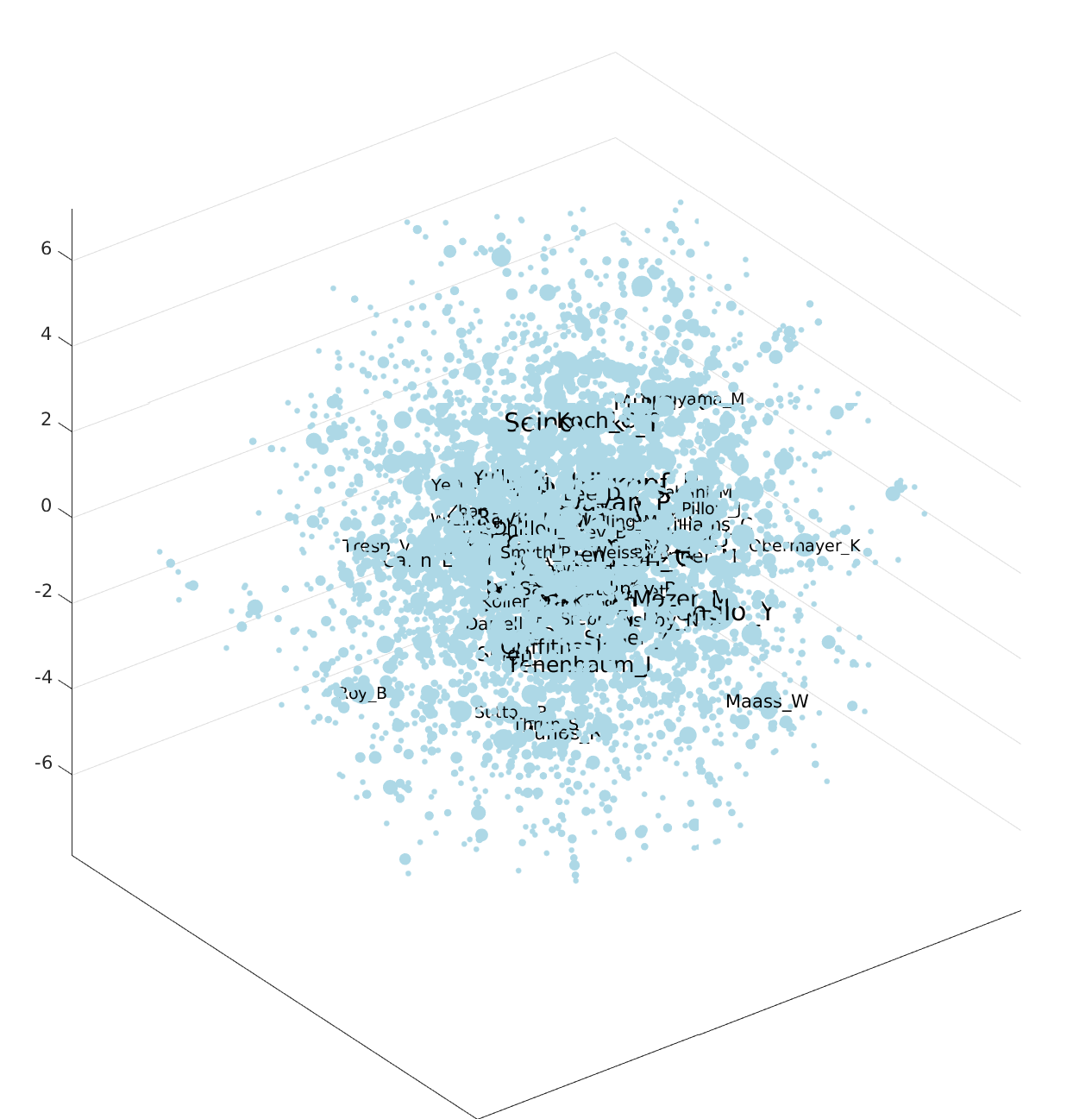}}}
	\vspace{\baselineskip}  
	
	\subcaptionbox{s-SNE 2D for \texttt{WorldTrade}}{
		\fbox{\includegraphics[height=\figwidth, width=\figwidth]{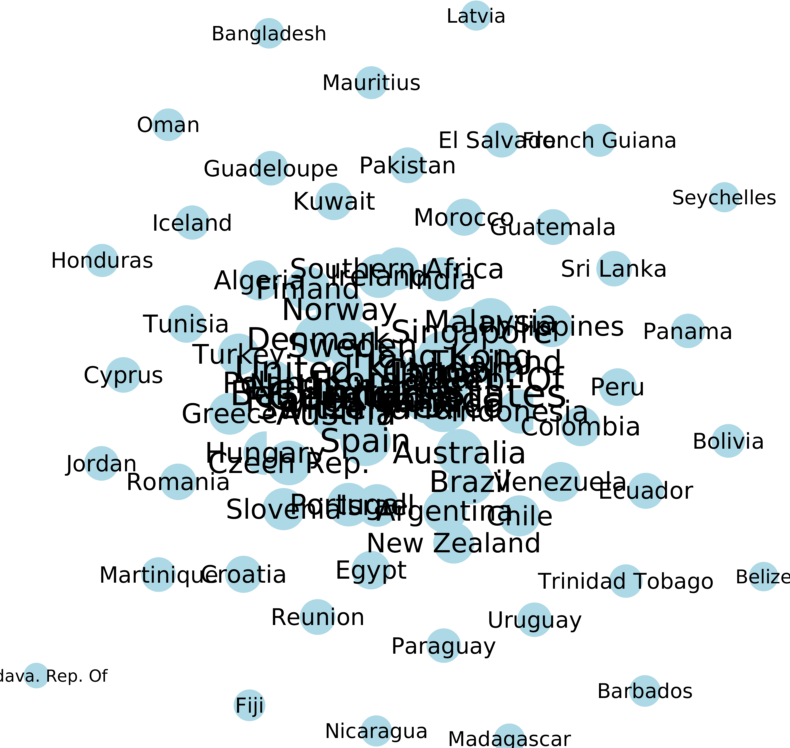}}}
	\subcaptionbox{s-SNE 3D for \texttt{WorldTrade}}{
		\fbox{\includegraphics[height=\figwidth,width=\figwidth]{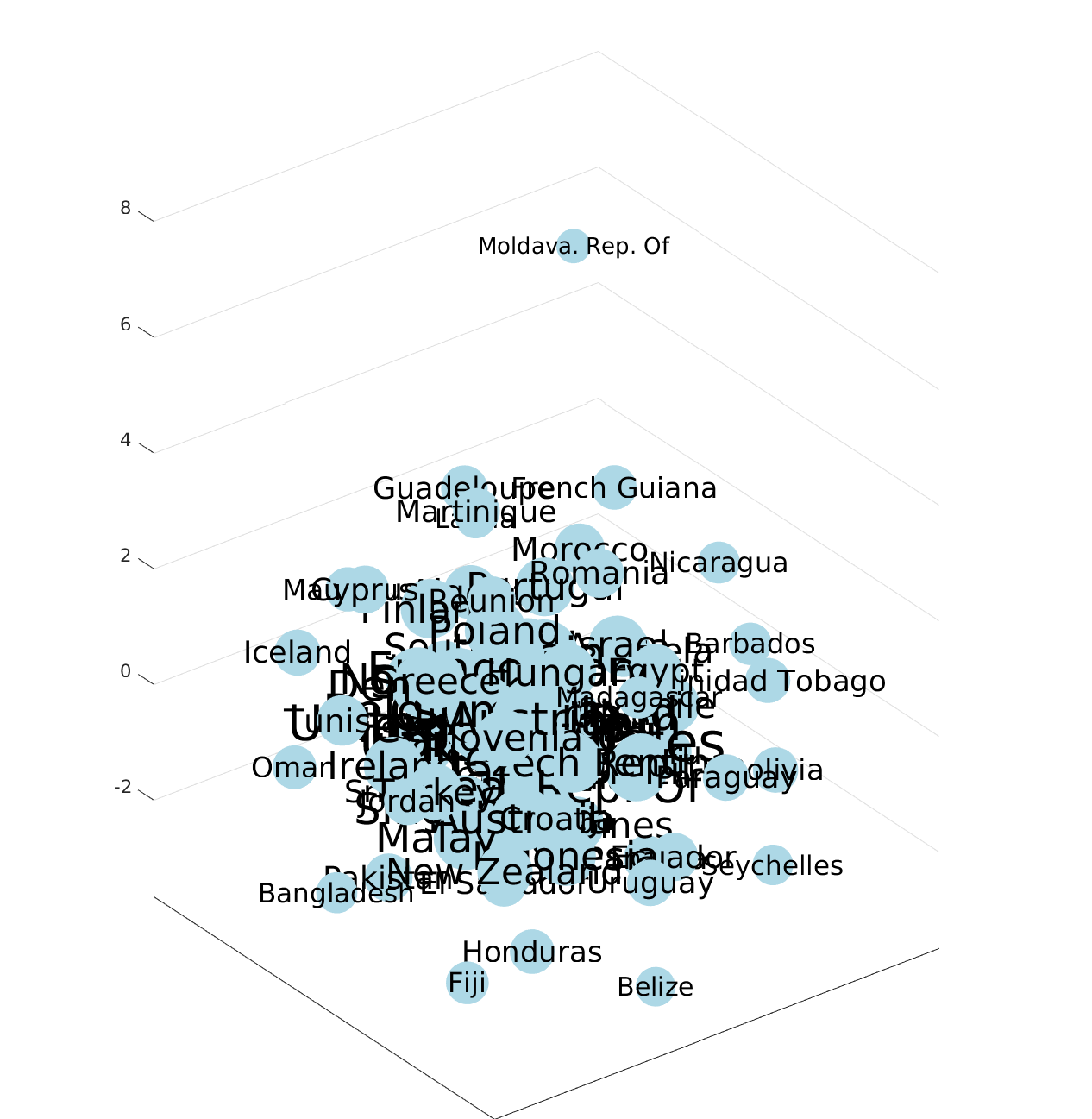}}}
	\vspace{\baselineskip}  
	
	\subcaptionbox{s-SNE 2D for \texttt{MIREX}}{
		\fbox{\includegraphics[height=\figwidth, width=\figwidth]{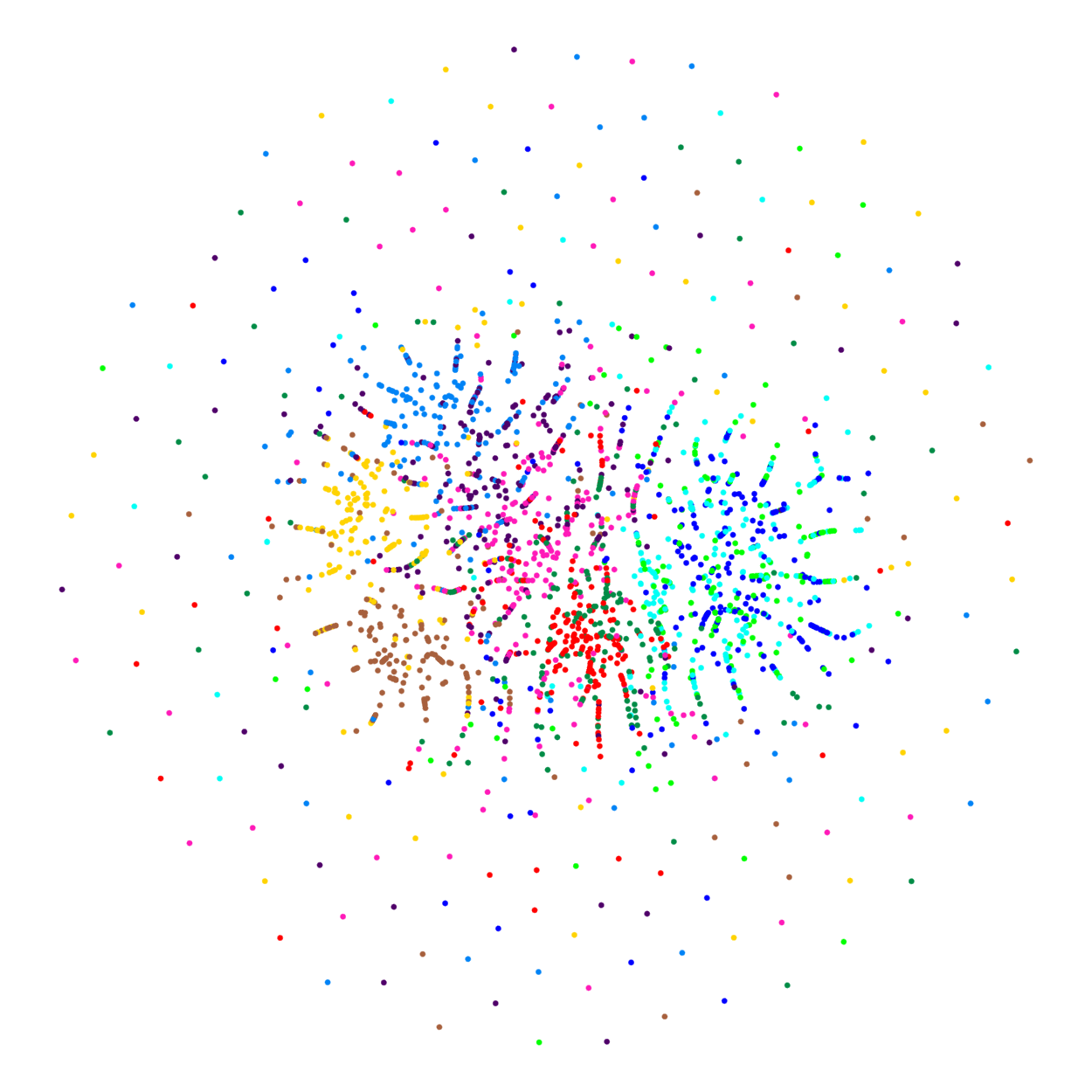}}}
	\subcaptionbox{s-SNE 3D for \texttt{MIREX}}{
		\fbox{\includegraphics[height=\figwidth,width=\figwidth]{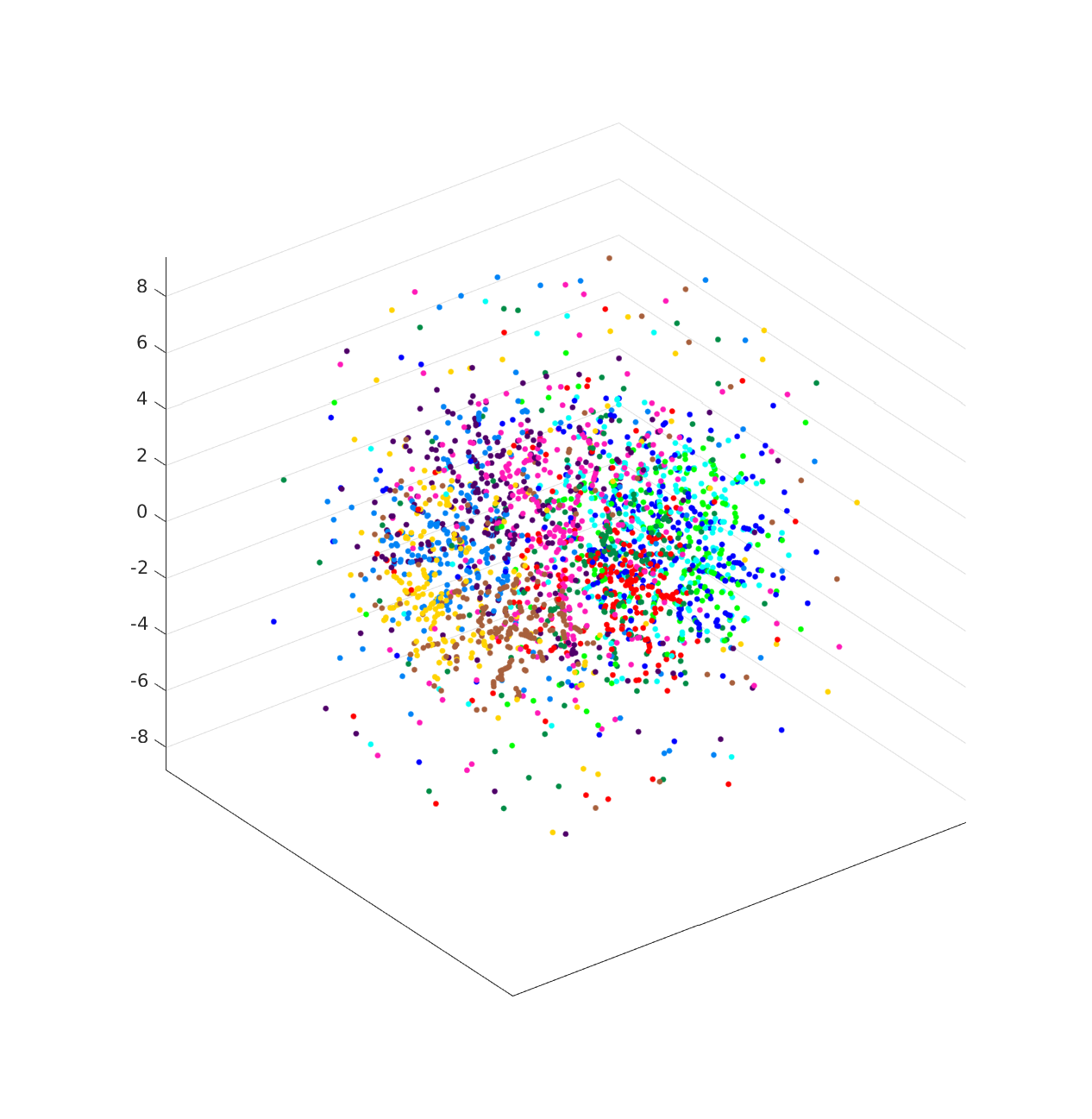}}}

	\caption{Visualizations of the \texttt{NIPS}, \texttt{WorldTrade}, and \texttt{MIREX} data sets using s-SNE with Gaussian kernel: (left) 2D layout and (right) 3D layout.}
	
	\label{fig:ssne}
	
\end{figure*}

\bibliographystyle{plainnat}
\bibliography{myrefs}

\end{document}